\documentclass[11pt]{article} 
\usepackage{fullpage}
\usepackage{natbib}
\usepackage{wrapfig}
\usepackage{amsthm}
\newtheorem{theorem}{Theorem} 
\theoremstyle{definition}
\newtheorem{definition}{Definition}
\usepackage{comment}
\usepackage[utf8]{inputenc} 
\usepackage[T1]{fontenc}    
\usepackage{hyperref}       
\usepackage{url}            
\usepackage{booktabs}       
\usepackage{tcolorbox}
\usepackage{nicefrac}       
\usepackage{microtype}      
\usepackage{xcolor}
\usepackage{subcaption}
\usepackage{caption}
\usepackage[normalem]{ulem}
\usepackage{tikz}
\usepackage{amsfonts}       
\graphicspath{{Figures/}}
\usepackage{amssymb}
\usepackage{amsmath}
\usepackage{amsthm}
\usepackage{mathtools}
\usepackage{wrapfig}
\usepackage{adjustbox}
\usepackage{multicol}
\usepackage{multirow}
\usepackage{amsthm}

\theoremstyle{definition}

\usepackage{color}
\usepackage{array}
\usepackage{algorithm}
\usepackage{algorithmic}

\usepackage[utf8]{inputenc}

\title{\textbf{Quasi-Framelets: Robust Graph Neural Networks via Adaptive Framelet Convolution}}

\author{Mengxi Yang\footnote{The University of Sydney}  \footnote{\texttt{mengxi.yang, dai.shi, xuebin.zheng, jie.yin, junbin.gao@sydney.edu.au}}\and Dai Shi\footnotemark[1]\and Xuebin Zheng\footnotemark[1] \and  Jie Yin\footnotemark[1] \and Junbin Gao\footnotemark[1]}

\date{}

\begin{document}

\maketitle
\begin{abstract}
This paper aims to provide a novel design of a multiscale framelet convolution for spectral graph neural networks (GNNs). While current spectral methods excel in various graph learning tasks, they often lack the flexibility to adapt to noisy, incomplete, or perturbed graph signals, making them fragile in such conditions. Our newly proposed framelet convolution addresses these limitations by decomposing graph data into low-pass and high-pass spectra through a finely-tuned multiscale approach. Our approach directly designs filtering functions within the spectral domain, allowing for precise control over the spectral components. The proposed design excels in filtering out unwanted spectral information and significantly reduces the adverse effects of noisy graph signals. Our approach not only enhances the robustness of GNNs but also preserves crucial graph features and structures. Through extensive experiments on diverse, real-world graph datasets, we demonstrate that our framelet convolution achieves superior performance in node classification tasks. It exhibits remarkable resilience to noisy data and adversarial attacks, highlighting its potential as a robust solution for real-world graph applications. This advancement opens new avenues for more adaptive and reliable spectral GNN architectures.
\end{abstract}

\section{Introduction}
Graph neural networks (GNNs), as a deep representation learning method for graph data, have aroused considerable research interest in the recent years. 
\citet{ DuvenaudMaclaurinIparraguirreBombarellHirzelAspuruGuzikAdams2015} and \citet{KipfWelling2017} designed the Graph Convolution Networks (GCNs) with a significantly improved performance on semi-supervised node classification. Since then, a large number of GNN variants have been proposed, such as GAT  \citep{velivckovic2017graph}, GraphSAGE \citep{HamiltonYingLeskovec2017}, GIN \citep{XuHuLeskovecJegelka2019}, and DGI \citep{VelickovicFedusWilliamHamiltonLioBengioDevonHjelm2019}. As a further development, more types of graph learning tasks have been explored. For example, \citet{ZouPengHuangYangLiWuLiuYu2023} considered structural
data learning based on GNNs while \citet{SunYePengWangYu2023} explored continual graph learning.

Recent works \citep{fu2020understanding, ma2021unified} provided a unified perspective about these state-of-the-art GNNs to set up a framework based on graph denoising view. This perspective demonstrates a promising potential for improving representation learning on graphs through efficiently handling adaptive smoothness across graphs. 
Graph convolution typically serves as a feature aggregation operator to exploit information in node neighborhoods to facilitate downstream tasks like node classification. For most GNNs, the main difference between them is in their various aggregation operators. In other words, from the graph signal denoising view, different GNNs differ in the way signals are regulated through convolution. Although various graph convolutions have been proposed, the majority of them are limited to work as a low-pass filter \citep{fu2020understanding}, which means that high-frequency graph information is abandoned. In signal processing, high-frequency signals tend to be treated as noises or uncommon information and thus, be removed to achieve global smoothness. On graphs, many state-of-the-art GNN models follow the same idea and adopt $l_2$-based graph smoothing when recasting them as solving a graph signal denoising problem \citep{liu2021elastic}. However, the uncommon but true graph signals are often ignored by these methods. In order to enhance the local smoothness adaptivity of GNNs, Elastic GNNs \citep{liu2021elastic} introduce a novel regularization term, considering both $l_2$ and $l_1$ norm.

In recent years, similar efforts have been made on spectral GNNs, such as ChebyNet \citep{defferrard2016convolutional},  \citep{KipfWelling2017}, and GWNN \citep{xu2019graph}. As a pioneering work of spectral approaches, spectral CNN leverages graph data with the graph Fourier transform to perform graph convolution. Subsequent works like \citep{KipfWelling2017} proposed a spectral method where a graph convolutional network is designed via a localized first-order approximation of spectral graph convolutions. To take into account more localized properties, graph wavelet transform was introduced in \citep{xu2019graph}. Nevertheless, 
all previous graph convolutions operate as low-pass filters on graphs, but ignore the necessity of high-pass filters. To achieve multi-resolution graph signal representation and analysis,  
\citet{dong2017sparse} developed tight graph framelets by combining tight wavelet frames and spectral graph theory, yielding an efficient and accurate representation of graphs. \citet{zheng2021framelets} further proposed UFG convolution that regulates low-pass and high-pass information from graph signals simultaneously 
to achieve better performance. Nonetheless, 
the framelet construction of all existing methods have to rely on the multiple finite filter banks in the spatial domain, thus lacking adaptivity in cutting off undesired local, high-frequency graph signals.

To fill this research gap, in this paper, we propose a novel framelet convolution called \textbf{Quasi-Framelets} for robust graph representation learning. The core idea is to directly construct filtering functions in the spectral domain so as to accurately model local signals on noisy graphs. This leaves room for designing any specific framelets as appropriate for meaningful frequency suppression. Specifically, Quasi-Framelets leverage framelet transforms to decompose graph data into low-pass and high-pass spectra in a fine-tuned multiscale way, which enables better preservation of node features and graph geometric information. 
Through directly constructing filtering functions in the spectral domain, our proposed framelet convolution offers greater flexibility in achieving local adaptive smoothness and cutting-off unwanted spectral information. This allows Quasi-Framelets to effectively alleviate the adverse effect of noisy graph signals and demonstrate desired robustness to noisy node attributes and disturbed graph topology,
In summary, our contributions are three-fold:
\begin{enumerate}

\item We propose a new method of designing multiscale framelets from spectral perspectives. Our new method offers learning capacity, adaptivity, and flexibility in filtering out unwanted spectral information on noisy graphs.

\item We theoretically prove that the proposed framelet convolution has all the theoretical properties of classic framelets such as the fast algorithm in decomposition and reconstruction. 

\item The results of extensive experiments validate the superior performance of our proposed framelet convolution on noisy graphs and under adversarial attacks for node classification tasks.
\end{enumerate} 

The rest of the paper is organized as follows. In Section \ref{Sec:2},
we review relevant recent work on graph neural networks and introduce related concepts and notation. In Section \ref{Sec:3}, we propose the new framework of construction of the so-called quasi-Framelets for graph signals, particularly the signal decomposition and reconstruction theory and algorithms are derived. Section \ref{Sec:4} is dedicated to conducting comprehensive experiments on six benchmark graph datasets and presenting the experimental results and analysis. Finally, this paper is concluded in Section \ref{Sec:5}.

\section{Preliminaries and Related Work}\label{Sec:2}
This section serves as a quick summary of the classic spectral GNNs and positions our work in the related literature. For this purpose, we first introduce the following commonly used notation.
Given an undirected graph $\mathcal{G} =(\mathcal{V}, \mathcal{E})$, $\mathcal{V} = \{v_n\}^N_{n=1}$ represents the set of $N$ nodes, and $\mathcal{E} = \{e_{ij} = (v_i, v_j)\}$ is the set of edges with cardinality $E=|\mathcal{E}|$. For graph $\mathcal{G}$, the adjacency matrix is denoted by $\mathbf A \in\mathbb{R}^{N\times N}$ such that $\mathbf A_{ij} = 1$ if $e_{i,j} = (v_i, v_j)$ is an edge of the graph {and $0$ otherwise.} Besides, the graph Laplacian matrix is defined as $\mathbf L = \mathbf{D} - \mathbf{A}$, where $\mathbf{D}$ is a diagonal degree matrix. The normalized Laplacian matrix is $\widehat {\mathbf L} = \mathbf I_N - \mathbf{D}^{-1/2}\mathbf A\mathbf{D}^{-1/2} =
\mathbf U\boldsymbol{\Lambda}\mathbf{U}^T$, with a diagonal matrix of all the eigenvalues (spectra) $\boldsymbol{\Lambda} = \{\lambda_1, ..., \lambda_N\}$ of $\widehat {\mathbf L}$ as well as $\mathbf{U} = \{u_1, ..., u_N\}$ as the matrix of eigenvectors. We note that in this paper, for convenience reasons, we only consider undirected, unweighted graphs. 

\subsection{Classic Graph Neural Networks (GNNs)}
It is well-known that Graph Neural Networks (GNNs) layers can be categorized into two major groups: spatial-based GNNs and spectral-based GNNs \citep{WuPanChenLongZhangYu2021}. 

The spatial-based methods define graph
convolutions based on a node's spatial relations. The implementation of most spatial-based methods relies on the so-called message-passing schemes. Spatial-based GNNs incorporate node connectivities as spatial information, and message passing is a core mechanism within GNNs that facilitates the aggregation of information between nodes, allowing for the learning of graph-structured data. These techniques are widely applied in various domains, including social network analysis, recommendation systems, and bioinformatics.  The typical examples include the GCN \citep{KipfWelling2017},  GraphSAGE \citep{HamiltonYingLeskovec2017}, GIN \citep{XuHuLeskovecJegelka2019} and Spatial Graph Attention Networks (SpGAT) \citep{yu2018spatio} to name a few.

On the other hand, spectral-based GNNs build on the theory of Fourier analysis for graph signals. The convolution operations of the GNNs in this type of method are carried out in the spectral domain for graph signals. With graph Fourier transform and convolution theorem, the classic vanilla spectral GNN layer \citep{defferrard2016convolutional} converts graph signals from the spatial domain to the spectral domain and conducts convolution operation. {Specifically, for the graph feature matrix  $\mathbf X \in \mathbb R^{N\times d}$,} the vanilla spectral GNN defines the graph convolution operator as
\begin{align}
    g_{\theta} \star \mathbf X = \mathbf{U} g_{\theta}(\boldsymbol{\Lambda}) \mathbf{U}^T \mathbf X \label{eq:1}{,}
\end{align}
where $g_{\theta}(\boldsymbol{\Lambda}) =\text{diag}(g_{\theta}(\lambda_1), \ldots, g_{\theta}(\lambda_N))$ with the full set of spectra related to the spectral bases,  and $\star$ is the standard convolution operator. Here, $g_{\theta}$ is a designated spectral filtering function, usually seen as a discretized diagonal matrix that can be learnable. Furthermore, it is known that for any graph signal $\mathbf x \in \mathbb R^N$, $\mathbf U^T\mathbf x$ is called the graph Fourier transform of $\mathbf x$. 

Various spectral GNNs differ in the way how graph signals are regulated through spectral filters in the spectral domain. To capture localized property of graph signals and relieve the high computational burden of graph Laplacian eigendecomposition in \citep{BrunaZarembaSzlamLeCun2014}, localized polynomial approximation \citep{KipfWelling2017} and the concept of graph wavelets \citep{xu2019graph, defferrard2016convolutional} have been introduced to define convolution operators. The spectral filters proposed by these methods serve as low-pass filters, which means that all high-frequency signals, including true uncommon signals, are filtered. To further model the locality property of graph spectra, graph framelets \citep{dong2017sparse, zheng2021framelets,han2022generalized} provide multiresolution analysis by considering both low-pass and high-pass signals to construct filtering functions. 

The idea of graph framelets is initially proposed in \citep{zheng2021framelets}, where multiresolution analysis enables the extraction of both local and global information of graph data. However, one main drawback of this forward framelet strategy is that it is unclear what high frequency can be removed and what low frequency can be retained. Furthermore, all existing spectral methods rely on the multiple finite filter banks in the spatial domain to construct wavelets. Hence, they are unable to cut off undesired local, high-frequency graph signals.

\subsection{Robustness on Graph Learning}

With GNN models applied in various domains, a growing number of research works have concerned the robustness of GNNs given noisy graphs. As shown in  \citep{zugner2019certifiable}, node classification with GNNs can be easily attacked by perturbations on node features and graph structures. Such vulnerabilities towards attacks are caused by the message-passing mechanism of most GNNs. To improve the robustness of graph-based methods, many works have been proposed from a wide range of perspectives such as adversarial training \citep{feng2019graph}, transfer training \citep{tang2020transferring}, and smoothing distillation \citep{chen2019adversarial}. Meanwhile, some algorithms, such as spectral clustering \citep{bojchevski2017robust}, focus on enhancing robustness with respect to noise. Basically, these methods follow the idea which is to smoothing the prediction around inputs over graph structures, with the objective of considering both local smoothness and global graph structures. 

For GNNs, recent studies \citep{ma2021unified, fu2020understanding} first note that the layer of many GNNs, while intrinsically performing denoising and smoothing on node features and graph structures, can be indeed explained as the solution of the specifically designed optimization problem. For example, the layer of the classic GCN \citep{KipfWelling2017} is the solution to the following graph denoising problem
\begin{align}
\min_{\mathbf{X}} \quad
& \mathcal{L} = \|\mathbf X - \mathbf F\| ^2_F + c \cdot \text{tr}(\mathbf X^T\mathbf L \mathbf X)
\label{eq:3}{,}
\end{align}
where {we denote $\mathcal L$ as the quantity of the lost function, and }the first term measures the similarity between observed values $\mathbf F$ and the ground-truth $\mathbf X$ {(or its embedded version i.e., $\mathbf {XW}$)} and the second term shows a Laplacian regularization {(also known as the Dirichlet energy)} for true value $\mathbf{X}$ smoothing across graph $\mathcal{G}$.

While most representative GNNs favor the $l_2$-based (i.e., the quadratic form of the node features) smoothing on graph signals as defined in \eqref{eq:3}, other studies are concerned with the robustness of the GNN models towards severe outlier noise corruption. To enhance robustness, Elastic GNNs \citep{liu2021elastic} introduce a novel graph signal estimator that combines the strengths of both $l_1$ and $l_2$-based graph smoothing in the design of GNNs. The work of Elastic GNNs serves a similar role to wavelet denoising methods on graphs \citep{dong2017sparse} in preserving the smoothness adaptivity. 

Rooted in deep learning, adversarial training has been successful
in learning robust classifiers. Graph Adversarial Training (GraphAT) \citep{feng2019graph} was proposed to improve the robustness and generalization of models learned on graphs. \citet{tang2020transferring} took the goal against poisoning attacks by exploring clean graphs. Similarly, based on generative adversarial networks, \citet{LiFuZhuPengWangSunYuHe2023} introduced a robust and generalized framework to learn robust and generalized representations for graph nodes. In \citep{ZhangZitnik2020}, a general algorithm, named GNNGUARD, has been proposed to defend against a variety of training-time attacks, which can be integrated into any GNN structures. By representing each node as a Gaussian distribution, a robust GNN can be trained as done in \citep{ZhuZhangCuiZhu2019}.  

It is worthwhile noting that almost all these models in enhancing robustness fall into the regime of spatial-based GNNs. In contrast, our proposed Quasi-framelets convolution takes a distinct learnable approach; we aim to directly learn filtering functions in the spectral domain to better achieve both local and global adaptive smoothness and to improve GNN's robustness. Note that the Elastic GNN version in the quasi-framelet domain has been recently explored in \citep{ShaoShiHanVasnevGuoGao2022,shirevisiting}.

\section{The Model based on Quasi-Framelets}\label{Sec:3}
In this section, we propose the Quasi-Framelets (QUFG) framework on graphs, which includes a new set of QUFG transform functions and a fast algorithm for the transformation.

\subsection{The Quasi-Framelet Transform}
Our idea follows the basic Framelet theory, however, we work backward. Rather than looking for scaling functions from the finite filter banks in the spatial domain, we aim to \textbf{directly construct the spectral modulation (or filtering) functions in the spectral domain}. The starting point is from the necessary conditions for signal decomposition (Fourier transform) and reconstruction (Inverse Fourier transform).

\begin{definition}[Filtering functions for Quasi-Framelets] We call a set of $K+1$ positive filtering functions, $\mathcal{F} = \{g_0(\xi), g_1(\xi), ..., g_K(\xi)\}$, Quasi-Framelet scaling functions if the following identity condition is satisfied:
\begin{align}
g_0(\xi)^2 + g_1(\xi)^2 + \cdots + g_K(\xi)^2 \equiv 1,\;\;\; \forall \xi \in [0, \pi] \label{eq:4}
\end{align}
such that $g_0$ descends from 1 to 0 and $g_K$ ascends from 0 to 1 as frequency increases over the spectral domain $[0, \pi]$. 
\end{definition}
Particularly $g_0$ aims to regulate the highest frequency while $g_K$ to regulate the lowest frequency, and the rest to regulate other frequencies in between.  

Under Definition 1, designing new filtering functions becomes more flexible in that we only need to find appropriate filtering functions, each of which has the capability of modulating specific/desired frequency components. This is significantly different from the principle in the classic framelet design, where one looks for valid framelets from appropriate scaling functions from finite filter banks, but without clear guidance on how to choose for specific frequency modulation purposes.

As examples, we propose two sets of such filtering functions.  
\begin{enumerate}
\item Sigmoid Filtering Functions ($K=1$):
\begin{align*}
        g_0(\xi)=\sqrt{1\!-\!\frac{1}{1\!+\!\exp\{-\alpha (\xi/\pi - 0.5)\}}};\;\;
        g_1(\xi)=\sqrt{\frac{1}{1\!+\!\exp\{-\alpha (\xi/\pi - 0.5)\}}}{,}
\end{align*}
where $\alpha>0$. In general, we suggest taking $\alpha \geq 20$ to ensure sufficient modulation power at both the lowest and highest frequencies.

\item Entropy Filtering Functions ($K=2$):
    \begin{align*}
    g_0(\xi) &= \begin{cases} \sqrt{1 - g^2_1(\xi)}, & \xi <= \pi/2\\
    0, & \text{otherwise}
    \end{cases}\;\;\;\;\;\;
        g_1(\xi) = \sqrt{4\alpha \xi/\pi - 4 \alpha \left(\xi/\pi\right)^2}\\
        g_2(\xi) &= \begin{cases} \sqrt{1 - g^2_1(\xi)}, & \xi > \pi/2\\
    0, & \text{otherwise}
    \end{cases}
    \end{align*}
    Where $0<\alpha\leq 1$ could be a learnable parameter. In our experiment, we empirically set $\alpha = 0.75$. When $\alpha=1$, $g^2_1(\pi\xi)$ is the so-called \textit{binary entropy function}.
\end{enumerate}

In the second example of Entropy Filtering Functions, we can clearly see that $g_0(\xi)$ could completely cut off higher frequency components while $g_2(\xi)$ could completely remove lower frequency components. The filtering in the final GCN could make such choices according to graph data in the learning process (recall that learnable coefficients will be assigned in the practical implementation). Both examples are shown in Fig.~\ref{fig1:quasi-framelet transformation matrices} with two choices of different $\alpha$ values. One can check that given a fixed range of (normalized) eigenvalues, $\alpha$ efficiently modulates the filtering results. Later on in Section \ref{sensitivity_on_alpha}, we empirically show that $\alpha$ owns the potential of mitigating the so-called heterophily adaption problem in graph neural networks.
\begin{figure*}[t]
\centering
\subfloat[Sigmoid with $\alpha = 10$]{
         \includegraphics[width=0.45\textwidth]{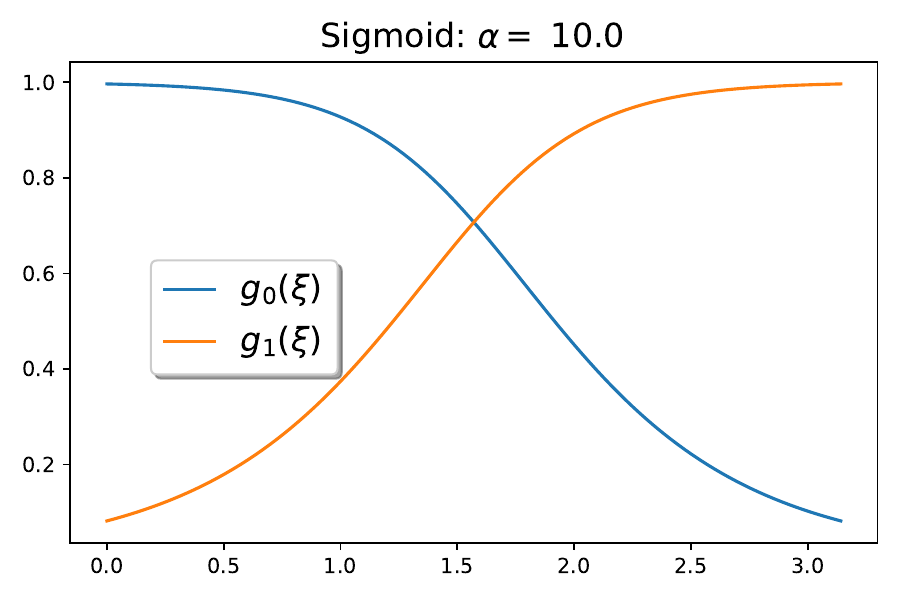}}\quad\quad
\subfloat[Sigmoid with $\alpha = 50$]{
         \includegraphics[width=0.45\textwidth]{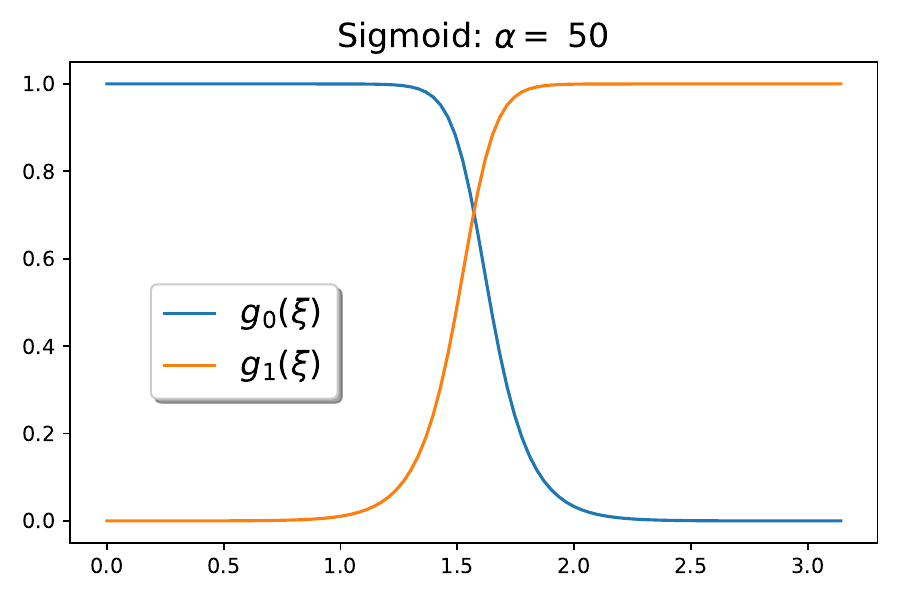}}\\~\\~\\
\subfloat[Entropy with $\alpha = 0.5$]{
         \includegraphics[width=0.45\textwidth]{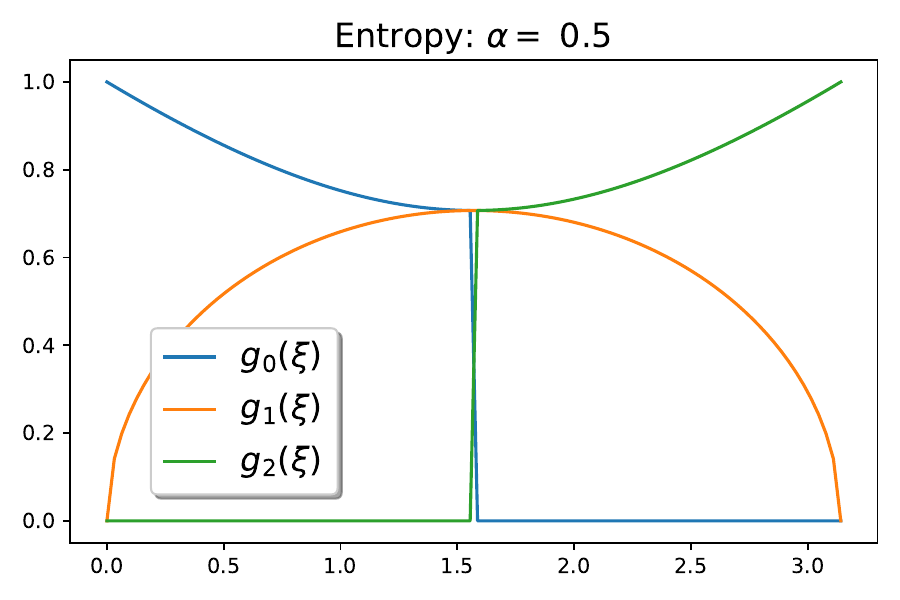}} \quad\quad
\subfloat[Entropy  with $\alpha = 0.1$]{
         \includegraphics[width=0.45\textwidth]{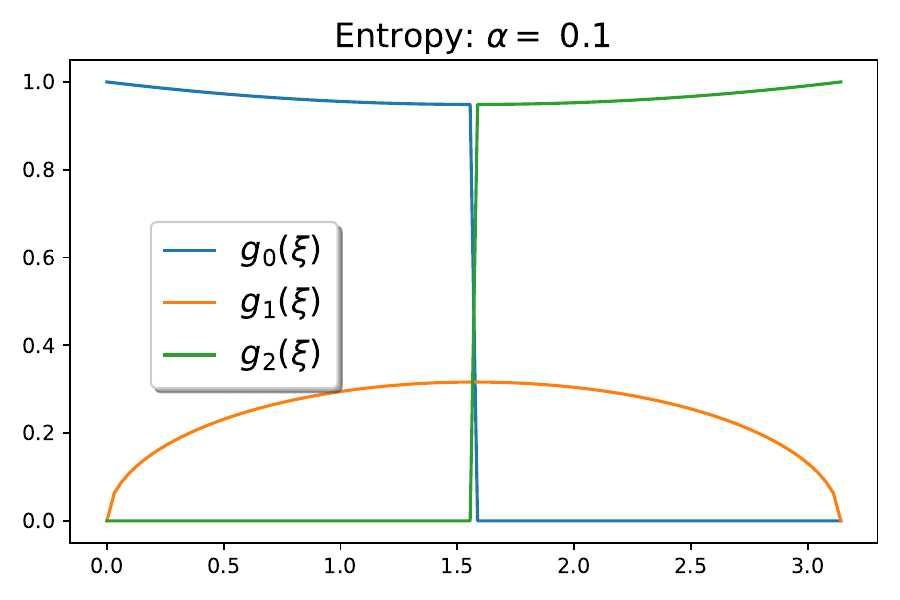}}       
\caption{ (a) and (b): examples of Sigmoid filtering functions with parameter $\alpha = 10$ and with $\alpha = 50$, respectively; 
(c) and (d): examples of Entropy filtering functions with parameter $\alpha  = 0.5$ and with $\alpha = 0.1$, respectively.}
\label{fig1:quasi-framelet transformation matrices}
\end{figure*}

Now consider a graph  $\mathcal{G} =(\mathcal{V}, \mathcal{E})$ and any graph signal $\mathbf x$ defined on its nodes. Suppose that $\mathbf U$ is the orthogonal spectral bases given by the normalized graph Laplacian $\widehat {\mathbf L}$ with its spectra $0\leq \lambda_1\leq \lambda_2 \leq \cdots\leq \lambda_N$. To build an appropriate spectral transformation defined as \eqref{eq:1}, for a given set of Quasi-Framelet functions $\mathcal{F} = \{g_0(\xi), g_1(\xi), ..., g_K(\xi)\}$ defined on $[0, \pi]$ and a given level $L$ ($\geq 0$), define the following Quasi-Framelet signal decomposition 
\begin{align}
\widehat{\mathbf x} = \mathcal{W}\mathbf x. \label{eq:9}
\end{align}
where the matrix operator $\mathcal{W} = [\mathcal{W}_{0,L}; \mathcal{W}_{1,0}; ...; \mathcal{W}_{K,0}; $ $\mathcal{W}_{1,1}; ..., \mathcal{W}_{K,L}]$ is stacked vertically in MATLAB notation from the following  signal transformation matrices
\begin{align}
    \mathcal{W}_{0,L} &= \mathbf U g_0\left(\!\frac{\boldsymbol{\Lambda}}{2^{m+L}}\!\right) \cdots g_0\left(\!\frac{\boldsymbol{\Lambda}}{2^{m}}\!\right) \mathbf U^T, \; 
    \mathcal{W}_{k,0} = \mathbf U g_k\left(\!\frac{\boldsymbol{\Lambda}}{2^{m}}\!\right) \mathbf U^T, \text{for } k = 1, ..., K, \label{eq:8b}\\
    \mathcal{W}_{k,\ell} &= \mathbf U g_k\left(\!\frac{\boldsymbol{\Lambda}}{2^{m+l}}\!\right))g_0\left(\!\frac{\boldsymbol{\Lambda}}{2^{m+\ell-1}}\!\right)) \cdots g_0\left(\!\frac{\boldsymbol{\Lambda}}{2^{m}}\!\right)) \mathbf U^T, 
    \;\;\text{for } k=1, ..., K, \ell = 1, ..., L.\label{eq:8}
\end{align}

Note that in the above definition, $m$ is called the coarsest scale level which is the smallest $m$ satisfying $2^{-m}\lambda_N \leq \pi$. Then, the quality of the reconstruction of the graph signal can be guaranteed based on the following theorem.

\begin{theorem}The Quasi-Framelet decomposition \eqref{eq:9} admits a perfect reconstruction for a given graph signal $\mathbf x \in \mathbb R^N$, that is $\mathbf x = \mathcal{W}^T\widehat{\mathbf x}, $ 
i.e., $\mathcal{W}^T\mathcal{W} = \mathbf I_N$.
\end{theorem}
\begin{proof}  According to the definition of all signal transformation matrices, all $\mathcal{W}_{i,j}$ are symmetric. Hence we have
\begin{align*}
& \mathcal{W}^T\mathcal{W}  =    
    \mathcal{W}_{0,L}\mathcal{W}_{0,L} +\sum^K_{k=1}\mathcal{W}_{k,L}\mathcal{W}_{k,L} + \sum^{L-1}_{l=1}\sum^K_{k=1}\mathcal{W}_{k,l}\mathcal{W}_{k,l}  \\
 =&\mathbf U g^2_0\left(\!\frac{\boldsymbol{\Lambda}}{2^{L+m}}\!\right) \cdots g^2_0\left(\!\frac{\boldsymbol{\Lambda}}{2^{m}}\!\right) \mathbf U^T   + \sum^K_{k=1}\mathbf U g^2_k\left(\!\frac{\boldsymbol{\Lambda}}{2^{m+L}}\!\right)g^2_0\left(\!\frac{\boldsymbol{\Lambda}}{2^{m+L-1}}\!\right) \cdots g^2_0(\frac{\boldsymbol{\Lambda}}{2^{m}}) \mathbf U^T\\
 &   + \sum^{L-1}_{l=}\sum^K_{k=1}\mathcal{W}_{k,l}\mathcal{W}_{k,l}\\
  =& \mathbf U \left(\!g^2_0\left(\!\frac{\boldsymbol{\Lambda}}{2^{L+m}}\!\right) +\sum^K_{k=1}g^2_k\left(\!\frac{\boldsymbol{\Lambda}}{2^{m+L}}\!\right)\!\right)g^2_0\left(\!\frac{\boldsymbol{\Lambda}}{2^{m+L-1}}\!\right) \cdots g^2_0\left(\!\frac{\boldsymbol{\Lambda}}{2^{m}}\!\right) \mathbf U^T\\
  & +\sum^{L-1}_{l=}\sum^K_{k=1}\mathcal{W}_{k,l}\mathcal{W}_{k,l} 
  \end{align*}
  \begin{align*}
  =& \mathbf U g^2_0\left(\!\frac{\boldsymbol{\Lambda}}{2^{m+L-1}}\!\right) \cdots g^2_0\left(\!\frac{\boldsymbol{\Lambda}}{2^{m}}\!\right) \mathbf U^T  +  \sum^{L-1}_{l=0}\sum^K_{k=1}\mathcal{W}_{k,l}\mathcal{W}_{k,l}\\
  & \vdots \\
  =& \mathbf U g^2_0\left(\!\frac{\boldsymbol{\Lambda}}{2^{m}}\!\right) \mathbf U^T + \sum^K_{k=1}\mathbf U g^2_k\left(\!\frac{\boldsymbol{\Lambda}}{2^{m}}\!\right) \mathbf U^T = \mathbf U\left(g^2_0\left(\!\frac{\boldsymbol{\Lambda}}{2^{m}}\!\right) +\sum^K_{k=1} g^2_k\left(\!\frac{\boldsymbol{\Lambda}}{2^{m}}\!\right)\right)\mathbf U^T\\
  =& \mathbf U\mathbf U^T = \mathbf I_N.
\end{align*}
where we have repeatedly used the condition \eqref{eq:3}.  This completes the proof. 
\end{proof}

\subsection{Graph Quasi-Framelets}
The Quasi-Framelet signal decomposition \eqref{eq:9} gives the signal decomposition coefficients. To better understand how the signal was decomposed on the framelet bases, we can define the following graph Quasi-Framelets which can be regarded as the signals in the spatial space.

Suppose $\{(\lambda_{i}, \mathbf u_{i})\}^N_{i=1}$  are the eigenvalue and eigenvector pairs for the normalized Laplacian $\widehat {\mathbf L}$ of graph $\mathcal{G}$ with $N$ nodes. {Denote by $\beta^k_{0}(\xi) = g_k(\frac{\xi}{2^{m}})$ and $\beta^k_{\ell}(\xi) =  g_k(\frac{\xi}{2^{m}}) g_0(\frac{\xi}{2^{m -1}}) \cdots g_0(\frac{\xi}{2^{m-\ell}})$ for $\ell = 1, 2, ..., L, k=0, 1, ..., K.$}
The Quasi-Framelets at scale level $\ell = 1,...,L$ for graph $\mathcal{G}$ with a given set of filtering functions $\mathcal{F} = \{g_0(\xi), g_1(\xi), ..., g_K(\xi)\}$ are defined, for $k = 1,...,K$, by
\begin{align} 
\begin{aligned}
\phi_{0,p}(q)&=\sum^N_{i=1}\beta^0_{L}\left(\!\frac{\lambda_i}{2^L}\!\right)\mathbf v_i (p) \mathbf v_i(q), \;\;\;  \psi^k_{0, p}(q) =\sum^N_{i=1}\beta^k_{0}\left(\!\frac{\lambda_i}{2^{0}}\!\right)\mathbf v_i (p) \mathbf v_i(q) \\
\psi^k_{\ell, p}(q) &=\sum^N_{i=1}\beta^k_{\ell}\left(\!\frac{\lambda_i}{2^{\ell}}\!\right)\mathbf v_i (p) \mathbf v_i(q), \;\;\;\; \ell = 1, 2, ..., L; \;\; k=1, 2, ..., K.
\end{aligned}
\end{align}
for all nodes $u, p$ and $\phi_{\ell, p}$ or $\psi^k_{\ell, p}$ is the low-pass or high-pass framelet translated at node $p$ at scale $\ell$ and node $p$, see \citep{dong2017sparse}.   

It is clear that the low-pass and high-pass framelet coefficients for a signal $\mathbf x$ on graph $\mathcal{G}$, denoted by $v_{\ell,p}$ and $w^k_{\ell, p}$, are the projections $\langle \phi_{\ell, p}, \mathbf x\rangle$ and $\langle \psi^k_{\ell, p}, \mathbf x\rangle$ of the graph signal onto framelets at scale $\ell$ and node $p$.  

Similar to the standard undecimated framelet system \citep{dong2017sparse}, we can define the Quasi-Framelet system QUFS as follows. For any two integers $L, L_1$ satisfying $L > L_1$, we define a Quasi-Framelet system $\text{QUFS} (\mathcal{F}; \mathcal{G})$ (starting from
a scale $J_1$) as a {non-homogeneous, stationary affine system}:
\begin{align}
    \text{QUFS}^L_{L_1}(\mathcal{F}; \mathcal{G})&:= \big\{\phi_{L_1, p}:p\in\mathcal{V}\big\} \cup \left\{\psi^k_{\ell,p}: p\in\mathcal{V}, \ell = L_1, ..., L\right\}^K_{k=1}{.} \label{Eq:new1}
\end{align}
 
The system $\text{QUFS}^L_{L_1}(\mathcal{F}; \mathcal{G})$ is then called a Quasi-tight frame for graph signal space $L^2(\mathcal{G})$ and the elements in $\text{QUFS}^L_{L_1}(\mathcal{F}; \mathcal{G})$ are called Quasi Tight Framelets on $\mathcal{G}$, \textit{quasi-framelets} for short. All the theories about the Quasi-Framelet can be guaranteed in the following theorem.

\begin{theorem}[Equivalence of Quasi-Framelet Tightness] Let $\mathcal{G} = (\mathcal{V},\mathcal{E})$ be a graph, $\{(\lambda_{i}, \mathbf u_{i})\}^N_{i=1}$ be the eigenvalue and eigenvector pairs for its normalized Laplacian $\mathbf L$, and   $\mathcal{F} = \{g_0(\xi), g_1(\xi), ..., g_K(\xi)\}$ be a set of Quasi-Framelet functions satisfying the identity condition \eqref{eq:4}. For $L>L_1$, $\text{QUFS}^L_{L_1}(\mathcal{F}; \mathcal{G})$ is the Quasi-Framelet system given in \eqref{Eq:new1}. Then the following statements are equivalent:

(i) For each $L_1 = 1,...,L$, the Quasi-Framelet system  $\text{QUFS}^L_{L_1}(\mathcal{F}; \mathcal{G})$  is a tight frame for $L^2(\mathcal{G})$, that is, $\forall \mathbf x \in L^2(\mathcal{G})$,
\begin{align}
\|\mathbf x\|^2 = \sum_{p\in\mathcal{V}}|\langle \phi_{L_1, p}, \mathbf x\rangle |^2 +\sum^L_{\ell=L_1}\sum^K_{k=1}\sum_{p\in\mathcal{V}} |\langle \psi^k_{\ell, p}, \mathbf x\rangle |^2.  \label{Them:1}
\end{align}
(ii) For all $\mathbf x \in L^2(\mathcal{G})$ and for $\ell = 1,...,L-1$, the following identities hold
\begin{align}
& \mathbf x  = \sum_{p\in\mathcal{V}}  \langle \phi_{L, p}, \mathbf x\rangle \phi_{L, p} + \sum^K_{k=1} \sum_{p\in\mathcal{V}} \langle \psi^k_{L, p}, \mathbf x\rangle \psi_{L, p}; \text{ and} \label{Them:2}\\
&\sum_{p\in\mathcal{V}}  \langle \phi_{\ell+1, p}, \mathbf x\rangle \phi_{\ell+1, p}  =  \sum_{p\in\mathcal{V}}  \langle \phi_{\ell, p}, \mathbf x\rangle \phi_{\ell, p}+ \sum^K_{k=1}\sum_{p\in\mathcal{V}}  \langle \psi^k_{\ell, p}, \mathbf x\rangle \psi_{\ell, p}{,} \label{Them:3}
\end{align}
(iii) For all $\mathbf x \in L^2(\mathcal{G})$ and for $\ell = 1,...,L-1$, the following identities hold
\begin{align*}
&\|\mathbf x\|^2 = \sum_{p\in\mathcal{V}}|\langle \phi_{L, p}, \mathbf x\rangle |^2 +\sum^K_{k=1}\sum_{p\in\mathcal{V}} |\langle \psi^k_{L, p}, \mathbf x\rangle |^2; \text{ and}
\\
&\sum_{p\in\mathcal{V}}  |\langle \phi_{\ell+1, p}, \mathbf x\rangle|^2   =  \sum_{p\in\mathcal{V}}  |\langle \phi_{\ell, p}, \mathbf x\rangle|^2+ \sum^K_{k=1}\sum_{p\in\mathcal{V}}  |\langle \psi^k_{\ell, p}, \mathbf x\rangle|^2{.}
\end{align*}
\end{theorem}
\begin{proof}
(i)$\Longleftrightarrow$(ii):  Let $\Phi_j := \text{span}\{\phi_{j,p}: p\in \mathcal{V}\}$ and $\Psi^k_j:=\text{span}\{\psi^k_{j,p}: p\in \mathcal{V}\}$. Define projections $\mathbf P_{\Phi_j}$, $\mathbf P_{\Psi^k_j}$, for $k=1, 2, ..., K$, by
$$
\mathbf P_{\Phi_j}(\mathbf x) := \sum_{p\in\mathcal{V}}\langle \phi_{j,p}, \mathbf x\rangle \phi_{j,p};\;\;\;\;
\mathbf P_{\Psi_j}(\mathbf x) := \sum_{p\in\mathcal{V}}\langle \psi^k_{j,p}, \mathbf x\rangle \psi^k_{j,p}; \;\;\;\; \forall \mathbf x \in L^2(\mathcal G)
$$
As $\text{QUFS}^L_{L_1}(\mathcal{F}; \mathcal{G})$ is a quasi-tight frame for $L^2(\mathcal G)$ for $J_1 = 1, . . . , J$, we obtain by polarization identity
\begin{align}
\mathbf x = \mathbf P_{\Phi_{J_1}}(\mathbf x) + \sum^J_{j=J_1}\sum^K_{k=1}\mathbf P_{\Psi^k_j}(\mathbf x)
= \mathbf P_{\Phi_{J_1+1}}(\mathbf x) + \sum^J_{j=J_1+1}\sum^K_{k=1}\mathbf P_{\Psi^k_j}(\mathbf x){,} \label{Proof:1}
\end{align}
for all $\mathbf x \in L^2(\mathcal G)$ and for all $J_1 = 1, . . . , J$. Thus, for $J_1 = 1, . . . , J-1$,
\begin{align}
   \mathbf P_{\Phi_{J_1+1}}(\mathbf x) = \mathbf P_{\Phi_{J_1}}(\mathbf x) + \sum^K_{k=1}\mathbf P_{\Psi^k_{J_1}}(\mathbf x){,}{,}  \label{Proof:2}
\end{align}
which is \eqref{Them:3}.   

Moreover, when $J_1 = J$, \eqref{Proof:1} gives \eqref{Them:2}. Consequently, (i)$\Longrightarrow$(ii).  On the other hand,  recursively using \eqref{Proof:2} gives
\begin{align}
 \mathbf P_{\Phi_{n+1}}(\mathbf x) = \mathbf P_{\Phi_{J_1}}(\mathbf x) + \sum^n_{j=J_1}\sum^K_{k=1} \mathbf P_{\Psi^k_{j}}(\mathbf x)     \label{Proof:3}
\end{align}
for all $J_1 \leq n \leq J-1$. Taking $n = J-1$ together with \eqref{Them:2}, we deduce \eqref{Proof:1}, which is equivalent to \eqref{Them:1}. Thus, (ii)$\Longrightarrow$(i).

(ii)$\Longleftrightarrow$(iii): The equivalence between (ii) and (iii) simply follows from the polarization identity. This completes the proof.  
\end{proof}

Finally, although the tightness property is one of the key factors for successfully incorporating the framelet system into the graph signals, recent studies \citep{han2023continuous,shao2023unifying} have also explored the multiscale GNNs under a wider range of filtering functions (i.e., without tightness property). Based on the above statement, the tightness property of framelet filtering functions ensures a perfect reconstruction of any given signal without information loss, although comparable learning accuracy has also been observed in  \citep{han2023continuous}. In this paper, our purpose is to deploy a well-defined, tightness, and robust quasi-framelet system into graphs. Therefore, we leave the comparison between the deeper mathematical properties between our methods and other multiscale GNNs in future work.

\subsection{Fast Algorithm for the Quasi-Framelet Transform and Its Computational Complexity Analysis}
\begin{figure}[t]
\centering
\includegraphics[width=0.6\textwidth]{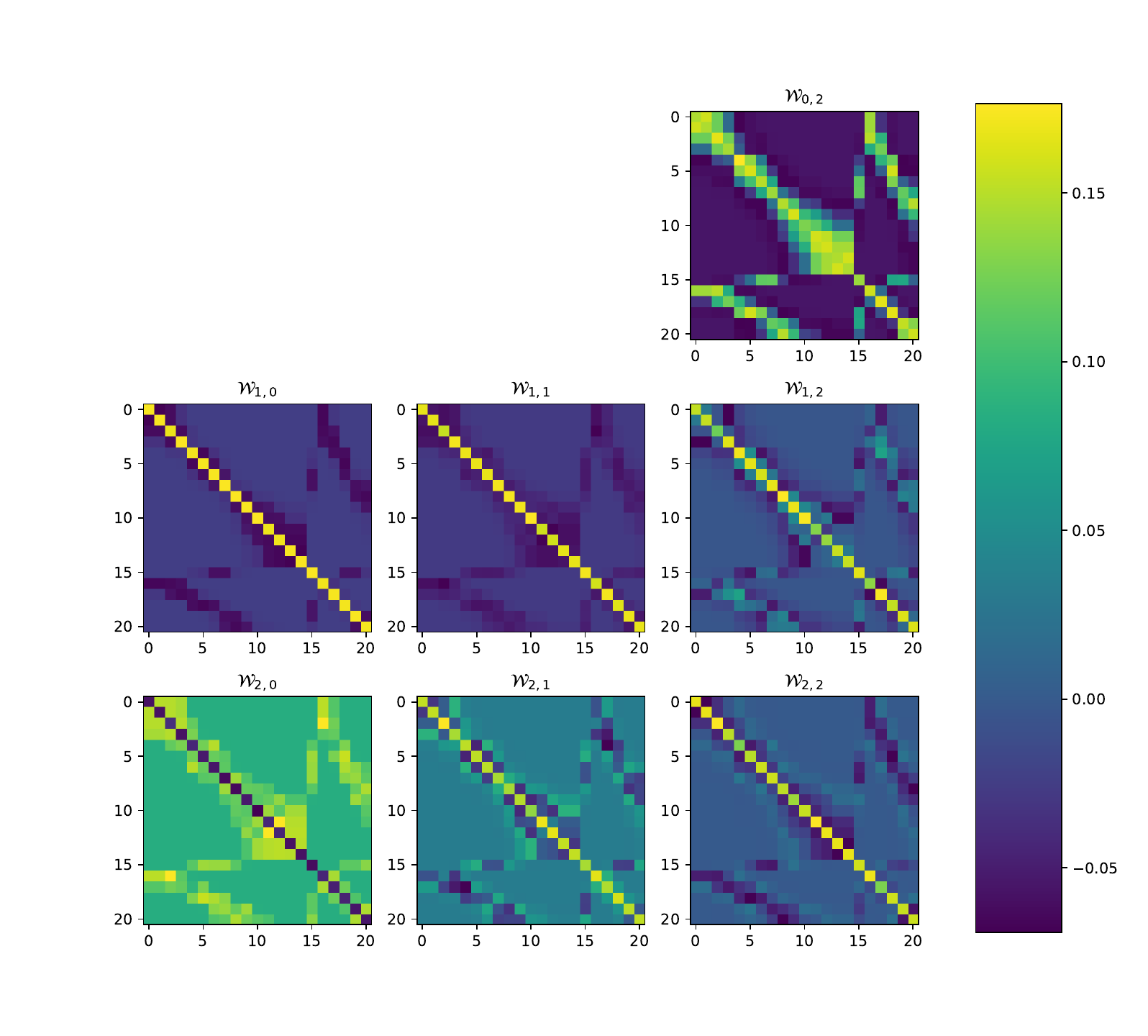}
  \caption{Quasi-Framelet transformation matrices $\mathcal{W}_{k,l}$ at different scales from the left column $l=0$ to the right column $l = L = 2$, for a graph with 21 nodes. The first row corresponds to the lowest frequency for the entropy filtering function $g_0(\xi)$, the middle row $g_1(\xi)$, and the third row $g_2(\xi)$,  based on \eqref{eq:Ta}-\eqref{eq:Tc}.}\label{fig2:quasi-framelet transformation matrices}
\end{figure}

Both the Quasi-Framelet signal decomposition $\widehat{\mathbf x} = \mathcal{W}\mathbf x$ and reconstruction $\mathbf x = \mathcal{W}^T\widehat{\mathbf x}$  are the building block for our Quasi-Framelet convolution. However, in its current form, it is computationally prohibitive for a large graph, as it would cost $O(N^3)$ to get the eigendecomposition of the Laplacian.

Similar to the standard spectral GNNs, we adopt a polynomial approximation to each filtering function $g_j(\xi)$ ($j=0, 1, ..., K$). We approximate $g_j(\xi)$ by Chebyshev polynomials $\mathcal{T}^n_j(\xi)$ of a fixed degree $n$ where the integer $n$ is chosen such that the Chebyshev polynomial approximation is of high precision. In practice, $n=3$ is good enough. For simple notation, in the sequel, we use $\mathcal{T}_j(\xi)$ instead of $\mathcal{T}^n_j(\xi)$. Then the Quasi-Framelet transformation matrices defined in \eqref{eq:Ta} - \eqref{eq:Tc} can be approximated by, {for  $k=1, ..., K, \ell = 1, ..., L$},
\begin{align}
    \mathcal{W}_{k,0} &\approx \mathbf U \mathcal{T}_k\left(\!\frac{\boldsymbol{\Lambda}}{2^{m}}\!\right) \mathbf U^T =  \mathcal{T}_k\left(\!\frac1{2^{m}}\mathbf L\!\right),\label{eq:Tb}\\
    \mathcal{W}_{0,L} &\approx \mathbf U \mathcal{T}_0\left(\!\frac{\boldsymbol{\Lambda}}{2^{L+m}}\!\right) \cdots \mathcal{T}_0\left(\!\frac{\boldsymbol{\Lambda}}{2^{m}}\!\right) \mathbf U^T = \mathcal{T}_0\left(\!\frac1{2^{L+m}}{\mathbf L}\!\right) \cdots \mathcal{T}_0\left(\!\frac{1}{2^{m}}{\mathbf L}\!\right), \label{eq:Ta} \\
    \mathcal{W}_{k,\ell} &\approx \mathbf U \mathcal{T}_k\left(\!\frac{\boldsymbol{\Lambda}}{2^{m+\ell}}\!\right)\mathcal{T}_0\left(\!\frac{\boldsymbol{\Lambda}}{2^{m+\ell-1}}\!\right) \cdots \mathcal{T}_0\left(\!\frac{\boldsymbol{\Lambda}}{2^{m}}\!\right) \mathbf U^T \notag\\
    &= \mathcal{T}_k\left(\!\frac{1}{2^{m+\ell}}\mathbf L\!\right)\mathcal{T}_0\left(\!\frac{1}{2^{m+\ell-1}}\mathbf L\!\right) \cdots \mathcal{T}_0\left(\!\frac{1}{2^{m}}\mathbf L\!\right), \label{eq:Tc}
\end{align}

As an example, we use the heatmap in Fig.~\ref{fig2:quasi-framelet transformation matrices} to show these approximation matrices.

Hence the Quasi-Framelet transformation matrices are approximated by calculating the matrix power of the normalized Laplacian $\widehat{\mathbf L}$. Thus for a graph signal $\mathbf x$, its framelet $\text{QUFS}^L_1(\mathcal{F},; \mathcal{G})$ (see \eqref{Eq:new1}) can be approximately calculated by the following iteration, which also gives the graph signal Quasi-Framelet decomposition \eqref{eq:9}.

\vspace{-0.2cm}
\begin{align}
\begin{aligned}
\text{Start with:}& \text{ for }  k = 1, ..., K, \\
&\Phi_{0,0}:=  \mathcal{T}_0\left(\!\frac{1}{2^{m}}\widehat
{\mathbf L}\!\right)\mathbf x; \;\; \Psi_{k,0}:=  \mathcal{T}_k\left(\!\frac{1}{2^{m}}\widehat{\mathbf L}\!\right)\mathbf x; \\
\text{for }\ell = 1, ..&., L \text{ and } k = 1, ..., K  \text{ do}:  \\
&\Phi_{0,\ell}:=   \mathcal{T}_0\left(\!\frac{1}{2^{\ell+m}}\widehat{\mathbf L}\!\right)\Phi_{0,\ell-1}; \;\; \Psi_{k,\ell}:=   \mathcal{T}_k\left(\!\frac{1}{2^{\ell+m}}\widehat{\mathbf L}\!\right)\Phi_{0,\ell-1}.  
\end{aligned}  \label{Decomposition1}
\end{align}

Similarly for any given Quasi-Framelet signal $\widehat{\mathbf x}$, represented by its framelet $\text{QUFS}^L_1(\mathcal{F},; \mathcal{G})$ i.e., in vector form $\{\Phi_{0,L}\}\cup \{\Psi_{k, \ell}\}^{K, L}_{k=1, \ell=1}$, the reconstruction $\mathbf x = \mathcal{W}^T\widehat{\mathbf x}$  
can be implemented in the following recursive algorithm:
\begin{align}
\begin{aligned}
    \text{for } \ell = L, ..., 1, \text{do:}&\;\;\;\;
     \Phi_{0, \ell-1} :=  \mathcal{T}_0\left(\!\frac{1}{2^{\ell+m}}\widehat{\mathbf L}\!\right)\Phi_{0,\ell} + \sum^K_{k=1} \mathcal{T}_k\left(\!\frac{1}{2^{\ell+m}}\widehat{\mathbf L}\!\right)\Psi_{k,\ell}{;} \\
    \text{Then:}& \;\;\; \mathbf x \approx \mathcal{T}_0\left(\!\frac{1}{2^{m}}\widehat{\mathbf L}\!\right)\Phi_{0,0} + \sum^K_{k=1} \mathcal{T}_k\left(\!\frac{1}{2^{m}}\widehat{\mathbf L}\!\right)\Psi_{k,0}.
\end{aligned}  \label{Reconstruction1}    
\end{align}

Next, we give a computational complexity analysis for the decomposition and reconstruction algorithms defined in \eqref{Decomposition1} and \eqref{Reconstruction1}. Consider a graph signal $\mathbf x$ on a graph with $N$ nodes and $E$ edges. Let $n$ be the order of Chebyshev polynomials $\mathcal{T}_k$ ($k=0, 1, ..., K$). Clearly all the terms $\displaystyle\mathcal{T}_k\left(\!\frac1{2^{\ell+m}}\widehat{\mathbf{L}}\!\right)$ in \eqref{Decomposition1} are in the sum of power terms $\widehat{\mathbf L}^{j}$. For any graph signal $\mathbf x$,  $\widehat{\mathbf L}^{j}\mathbf x$ can be calculated as the iteration $\widehat{\mathbf L}(\cdots (\widehat{\mathbf L}\mathbf x)).$  In practice, we compute $\widehat{\mathbf L}\mathbf u$ using a sparse matrix multiplication with a complexity of $O(E)$. Hence, the total complexity for calculating $\displaystyle\mathcal{T}_k\left(\!\frac1{2^{\ell+m}}\widehat{\mathbf{L}}\!\right)$ is $O(dE)$. Finally, the overall complexity for  $\text{QUFS}^L_1(\mathcal{F},; \mathcal{G})$ in \eqref{Decomposition1} is $O(dL(K+1)E) \approx O(dE)$. Similarly the complexity of \eqref{Reconstruction1} is also approximately $O(dE)$.

\subsection{Quasi-Framelet Convolution}

Similar to classic spectral graph convolutions, we can easily define our Quasi-Framelet convolutions. 
Suppose the input feature matrix $\mathbf X \in \mathbb{R}^{N \times d}$ of the graph $\mathcal{G}$, i.e., the signal is in $d$ channels. Then the graph Quasi-Framelet convolution can be {defined} as:
\begin{align} 
g_{\theta} \star \mathbf X {:=} \sigma(\mathcal{W}^T g_{\theta}\circ (\mathcal{W}\mathbf X\mathbf W)){,} \label{Eq:Convolution}
\end{align} 
where the diagonal $g_{\theta}$ represents a filter to be applied on the spectral coefficients of the signal, $\mathbf W\in \mathbb{R}^{d\times d_1}$ is the feature transformation matrix and $\sigma$ is the layer activation function such as a \texttt{ReLu}. One can interpret the convolution as the following steps. Initially, the node feature $\mathbf X$ is embedded into a lower dimensional space by using matrix $\mathbf W$, and then a learnable spectral filtering process is conducted by first decomposing the embedded graph signal into the domain constructed by our proposed filtering functions. After learning the coefficients ($g_{\theta}$), we do the reconstruction process to project the signals back to the original domain. This whole process admits a convolution operation on the graph node signals. Accordingly, the quasi-framelet convolution can be expressed as follows:
\begin{align}
\label{convolution_explicit}
\mathbf H(r + 1) &= \mathcal W_{0,L}^T {\rm diag}(\mathbf \theta_{0,L}) \mathcal W_{0,L} \mathbf H(r) \mathbf W^r + \sum_{k,\ell} \mathcal W_{k,\ell}^T {\rm diag}(\mathbf \theta_{k,\ell}) \mathcal W_{k,\ell} \mathbf H(r) \mathbf W^r,
\end{align}
which we let $r$ be the number of layers, and $\mathbf H$ be the graph feature representation with $\mathbf H(0) = \mathbf X$. In the real implementation, one can first embed the input feature matrix $\mathbf H$ into its embedding space by leveraging a fully connected feed-forward neural network (i.e., $\mathrm{MLP}$). Then, quasi-framelet decomposition matrices ($\mathcal W$) are prepared through the Chebyshev polynomial of order $n$, and these decomposition matrices are then multiplied by the embedded node features to project the feature into the quasi-framelet domain. Filtering processes are then conducted by assigning each domain a learnable diagonal matrix (i.e., $\mathrm{diag}(\theta)$). Lastly, we project the filtered signal back to the original domain by multiplying it with $\mathcal W^T$.

One can check the cost of conducting quasi-framelet convolution \eqref{Eq:Convolution} can be estimated as the sum of $O(Ndd_1)$ (for $\mathbf X\mathbf W$), $O(N)$ (the filtering $g_{\theta}$), $2 O(d_1E)$ (for both $\mathcal{W}$ and $\mathcal{W}^T$). Hence the overall cost is roughly $O(Ndd_1)+O(d_1E)$.

\subsection{Learnable Diagonal Matrices Control Model's Adaptation Power}\label{sec:adaption_power}
Followed by the feature propagation scheme of quasi-framelet proposed in \eqref{convolution_explicit}, in this section, we briefly discuss the model's adaption power to different types of graphs through the lens of the learnable filtering matrices, i.e., $\mathrm{diag}(\theta)$ in \eqref{convolution_explicit}. We note that the term 
``different types of graphs'' here means the graph with different homophily indices\footnote{The level
of homophily of a graph can be measured by $\mathcal{H(G)} = \mathbb{E}_{i \in \mathcal{V}}[|\{j\}_{j \in \mathcal{N}_{i,y_j= y_i}}|/|\mathcal{N}_i|]$, where $|\{j\}_{j \in \mathcal{N}_{i,y_j= y_i}}|$ denotes the number
of neighbors of $i \in V$ that share the same label as $i$ such that $y_i = y_j$, and $\mathcal{H(G)} \rightarrow 1$ corresponds to strong homophily
while $\mathcal{H(G)} \rightarrow 0$ indicates strong heterophily. }. The labeling distributions within connected nodes are largely different between the graphs with different homophilic indices, and such feature of the graph requires the corresponding GNNs to be able to induce both \textbf{smoothing and sharpening effects} on the node features to fit \citep{di2022graph,han2022generalized}.
For the analysis convenience, we set $K =1$, which suggests there will be only one low-pass filtering and high-pass filtering function in the quasi-framelet system. We also omit the effect of the activation function as the generalization. Our result is straightforward. Accordingly, the propagation rule of our model can be reduced to
\begin{align}
\mathbf H(r + 1) &= \mathcal W_{0,1}^T {\rm diag}(\mathbf \theta_{0,1}) \mathcal W_{0,1} \mathbf H(r) \mathbf W^r +  \mathcal W_{1,1}^T {\rm diag}(\mathbf \theta_{1,1}) \mathcal W_{1,1} \mathbf H(r) \mathbf W^r,
\end{align}
which can be further expressed as:
\begin{align}
\mathrm {vec}(\mathbf H(r + 1)) &= \left(\mathbf W^r\otimes (\theta_{0,1} \mathcal W^T_{0,1} \mathcal W_{0,1} +\theta_{1,1} \mathcal W^T_{1,1} \mathcal W_{1,1}) \right) \mathrm{vec}(\mathbf H(r)) \notag \\
& = \sum_{s,i} \Big(\lambda_s^{\mathbf W^r}  \big(  \theta_{0,1} g^2_0(\lambda_i) + \theta_{1,1} g^2_1(\lambda_i) \big) \Big) c_{s,i}(r) \boldsymbol{\phi}_s^{\mathbf W^r} \otimes \mathbf u_i,
\end{align}
where we further assume that $\mathbf W^r$ is symmetric and let $\{\mathbf u_i, \lambda_i\}^N_{i=1}$ and $\{\boldsymbol{\phi}^{\mathbf W^r}_s, \lambda_s^{\mathbf W^r}\}^{S}_{s=1}$ be the eigenpairs of $\widehat{\mathbf L}$ and $\mathbf W^r$, respectively, and $c_{s,i}(r) = \langle  \mathrm{vec}(\mathbf H(r)), \boldsymbol{\phi}_k^{\mathbf W^r} \otimes \mathbf u_i \rangle$. Since for $\lambda \in [0,\pi]$, the filtering function $g^2_0(\lambda)$ and $g^2_1(\lambda)$ is monotonically decreasing and increasing, respectively, therefore when $\theta_{0,1} \gg\ \theta_{1,1}$\footnote{To further simplify the analysis, we assume both the entries of $\mathrm{diag}(\theta_{0,1})$ and $\mathrm{diag}(\theta_{1,1})$ are constants.} for all $i$, the feature propagation of quasi-framelet is dominated by the result of low-pass filter, resulting in a smoothing effect on the node features, which is more suitable for the graphs with higher homophily indices. On the other hand, when $\theta_{1,1} \gg\ \theta_{0,1}$ for all $i$, the model is dominated by the results from high-pass filtering functions, hence the graph signals tend to be distinguished from each other for fitting the graphs with low homophily indices. Accordingly, our quasi-framelet owns the power of fitting different types of graphs, and there shall be a potential manually-adjusted scheme on $\theta$ for this matter to further reduce the computational cost \citep{shi2023curvature,han2022generalized}. We leave this to one of the future works. 

\section{Experiments}\label{Sec:4}

\subsection{Datasets and Experimental Settings}

\textbf{Datasets.} 
To evaluate the effectiveness of the proposed model QUFG, we conduct experiments on six real-world graph datasets, including three citation networks (Cora, Citeseer, and  
PubMed) two coauthor networks\footnote{\small\url{https://github.com/shchur/gnn-benchmark/tree/master/data/npz}}, including Coauther-CS and Coauther-Physics, as well as one politic blog network (Polblogs)\footnote{\small\url{https://netset.telecom-paris.fr/datasets}}.  
The statistics of the six datasets are shown in Table~\ref{Table1}. {We further highlight that although our proposed QUFG model can be easily extended to other graph learning tasks such as graph pooling and link classification, these tasks are not the main targets that our model wants to improve since, for a given graph, we aim to efficiently filter those unwanted graph signal components by spectral filtering. Hence, our method is more suitable for the node-feature-level classification rather than graph and link level classification and their noised versions. Accordingly, in this paper, we only focus on the node classification tasks, and we leave both empirical and theoretical works on the potential framelet graph pooling and link classification to future works.}

\begin{table}[pth]
\caption{The statistics of graph datasets} \label{Table1}
\centering
\tabcolsep 2.5pt
\begin{tabular}{|c|cccc|}
\hline
Graphs  & Nodes & Edges & Features & \#Classes  \\
\hline
\textbf{Cora} & 2,708 & 5,429 & 1,433 & 7   \\
\textbf{Citeseer} & 3,327 & 4,732 & 3,703 & 6 \\
\textbf{Pubmed} & 19,717 & 44,338 & 500 & 3 \\ \hline
\textbf{Coauthor-CS} & 18,333 & 163,788 & 6,805 & 15 \\
\textbf{Coauthor-Physics} & 34,493 & 495,924 & 8,415 & 3 \\
\textbf{Polblogs} & 1,490 & 19,025 & index & 2\\
\hline
\end{tabular}
\end{table}

\textbf{Baselines.} For assessing the robustness towards noises on node features, our QUFG model is compared with several competitive baselines. 1) Spectral GNNs: Spectral CNN \citep{BrunaZarembaSzlamLeCun2014}, Chebyshev \citep{defferrard2016convolutional}, GWNN \citep{xu2019graph}, UFG \citep{zheng2021framelets}; 2) Spatial GNNs: GCN \citep{KipfWelling2017}; For robustness evaluation with the adversarial attacks on graph structures, we compare the QUFG model with Spectral CNN \citep{BrunaZarembaSzlamLeCun2014}, Elastic GNNs \citep{liu2021elastic}, UFG \citep{zheng2021framelets} and GCN \citep{KipfWelling2017}, and SOTA baselines H2GCN \citep{ZhuYanZhaoHeimannAkogluKoutra2020} and GPRGNN \citep{ChienPengLiMilenkovic2021}. We also include the framelet p-Laplacian model (pL-UFG) proposed in \citep{ShaoShiHanVasnevGuoGao2022} since the model can be considered as the elastic GNN in the quasi-framelet domain.

\textbf{Setup.} 
For each experiment, we report the results of the QUFG model using entropy filtering functions since their performance is better than that of QUFG with sigmoid functions. The results are averaged over 10 runs. Each dataset is split into training, validation, and test sets by $20\%$, $20\%$, $60\%$. Except for some default hyper-parameters, such as learning rates and the number of hidden units, all models are tuned to the validation set. For Cora, Polblog, Coauthor graphs, and Citeseer, the number of hidden units is 32, the learning rate is 0.01, and the number of iterations is 200 with early stopping on the validation set. For PubMed, experiments are conducted with 64 hidden units, a learning rate of 0.005, and 250 iterations. The code used in our experiments is publicly
available at \url{https://github.com/mengxang/quasi-framelet}. All experiments are conducted in Python 3.11 with Pytorch Geometric on one NVIDIA\textsuperscript{\textregistered}RTX 3060 GPU with 3584 CUDA cores and 12GB GDDR6 memory size.  

\setcounter{table}{1}
\begin{table}[t]
\caption{{Node classification accuracy (\%) under binary noise settings. ART (ms): Average Running Time per epoch in ms. The best and second-best learning accuracy are in bold and underlined. We also note that for the p-Laplacian framelet model, we fixed $p =2.5$ so that the model is more suitable for citation networks \citep{ShaoShiHanVasnevGuoGao2022}. We also highlight that our results on pL-UFG is different compared to \citep{ShaoShiHanVasnevGuoGao2022} due to the difference in data split.}}
\label{table:table 6a}
\setlength{\tabcolsep}{7.5pt}
\renewcommand{\arraystretch}{1.5}
\centering
\begin{adjustbox}{width= 12cm,center}
\begin{tabular}{|c|c |c c c c c c|}
\hline
Datasets & Noise Level (\%) & Clean & 5  & 15 & 25 & 50  &ART (ms)\\
\hline
\multirow{6}{*}{\textbf{Cora}} & Spectral CNN 
& 73.3 & 26.0$\pm$5.1  & 19.5$\pm$4.1  & 15.2$\pm$4.7 & 19.1$\pm$6.9   & 12.3\\
& ChebyNet 
& 81.2 & 30.6$\pm$4.4 & 30.9$\pm$2.6 & 21.5$\pm$7.6 & 18.1$\pm$4.1 &  16.2\\
& GCN 
& 81.5 & 76.7$\pm$0.8  & 67.8$\pm$1.0 & 61.0$\pm$0.9 & 54.5$\pm$2.4 &   9.25 \\

& GWNN 
& 82.8 & 43.5$\pm$1.6  & 30.2$\pm$1.0 & 23.6$\pm$1.3 & 20.1$\pm$3.2 &15.2\\

& GPRGNN 
& 81.6\(\pm\)0.7    &65.9\(\pm\)7.2  &24.1\(\pm\)5.3  &22.4\(\pm\)7.0  &21.5\(\pm\)6.1 &14.2 \\

& H2GCN 
& 82.9\(\pm\)0.5    &70.3\(\pm\)1.3  &34.6\(\pm\)5.6  &26.1\(\pm\)7.5  &24.1\(\pm\)8.2  &26.5\\
& UFG 
& 83.6$\pm$0.6 & 75.6$\pm$0.8  & \underline{70.6$\pm$0.7} & {66.9$\pm$1.7} & {60.6$\pm$2.7} & 29.4\\

& pL-UFG$^{2.5}$ &\underline{83.9$\pm$1.7}  &\underline{77.8\(\pm\)2.4}  & \underline{68.8\(\pm\)0.3}  &\textbf{68.4\(\pm\)0.9}  &\underline{66.5\(\pm\)1.6}  &40.8  \\

& QUFG & \textbf{83.9$\pm$0.8} & \textbf{79.1$\pm$0.9} & \textbf{71.6$\pm$1.2} & \underline{68.2$\pm$2.1}& \textbf{62.7$\pm$2.1} & 32.7\\
\hline

\multirow{6}{*}{\textbf{Citeseer}} & Spectral CNN 
& 58.9 & 19.0$\pm$1.4 & 17.6$\pm$3.9 & 18.0$\pm$2.6 & 20.3$\pm$4.5  &9.3\\
& ChebyNet 
& 69.8 & 19.4$\pm$1.8 & 17.3$\pm$4.1  & 18.2$\pm$2.7 & 17.2$\pm$3.3 & 31.4\\
& GCN 
& 70.3 & \underline{62.6$\pm$0.8} & {47.3$\pm$1.2}  & {39.4$\pm$1.1} & {31.9$\pm$2.5} & 10.1\\
& GWNN 
& 71.7 & 49.4$\pm$1.4 & 21.1$\pm$0.3  & 21.3$\pm$0.5  & 15.8$\pm$0.9 &16.5\\

& GPRGNN 
 & 70.9\(\pm\)1.5 & 19.3\(\pm\)2.1 & 18.6\(\pm\)1.7 & 18.4\(\pm\)2.0   & 18.1\(\pm\)2.3    &  13.6\\

&  H2GCN  & \underline{ 73.1\(\pm\)0.4    } & 47.8\(\pm\)2.2   & 24.2\(\pm\)1.3   &19.5\(\pm\)1.5   & 19.2\(\pm\)0.9   & 26.3\\
& UFG 
& 72.7$\pm$0.6 & {61.6$\pm$1.3} & 40.8$\pm$7.0  & 26.0$\pm$3.2 & 23.5$\pm$2.7  & 49.1 \\

& pL-UFG$^{2.5}$ &72.2\(\pm\)1.4  &\textbf{64.8\(\pm\)3.1}  & \textbf{50.3\(\pm\)3.3}  &\underline{42.8\(\pm\)2.8}  &\underline{39.7\(\pm\)1.4} & 50.1 \\

& QUFG  & \textbf{73.4$\pm$0.6} & 60.9$\pm$1.8 & \underline{47.4$\pm$1.8} & \textbf{44.9$\pm$2.3} & \textbf{42.4$\pm$1.6}   & 45.6\\

\hline

\multirow{6}{*}{\textbf{Coauthor-CS}} & Spectral CNN 
& 54.3$\pm$0.5 & 12.3$\pm$4.0 & 8.5$\pm$4.4 & 6.1$\pm$2.7 & 6.1$\pm$2.3  &33.6\\
& ChebyNet 
& 61.4$\pm$0.9 & 14.5$\pm$4.5 & 10.1$\pm$2.3 & 8.3$\pm$1.3& 5.0$\pm$1.5  &190.6\\
& GCN 
& 91.3$\pm$0.2 & 89.3$\pm$0.6 & \underline{84.5$\pm$1.1}  & \underline{64.9$\pm$0.7} & 8.3$\pm$3.0   & 32.4\\
& GWNN
& 89.5$\pm$2.3 & 80.8$\pm$0.2 & 71.5$\pm$1.3 & 46.3 $\pm$1.1 & 9.2$\pm$5.4 & 49.1\\
&  GPRGNN  & 92.4\(\pm\)0.2  & 88.1\(\pm\)1.0  & 32.6\(\pm\)12.4 & 20.4\(\pm\)6.2 & 10.7\(\pm\)2.9 &  30.2\\
& H2GCN  & \underline{ 93.0\(\pm\)0.1    } & 71.3\(\pm\)0.7 & 12.7\(\pm\)5.6   & 11.9\(\pm\)6.4   & 10.2\(\pm\)4.8  &  147.2\\
& UFG 
& 91.9$\pm$0.3 & {90.2$\pm$0.5} & 80.8$\pm$1.5  & 60.8$\pm$0.4 & \underline{9.4$\pm$1.5}  &411.7\\

& pL-UFG$^{2.5}$ &92.4\(\pm\)0.1  &\textbf{91.1\(\pm\)0.8}  &81.4\(\pm\)2.3  &66.0\(\pm\)5.3  &\textbf{14.8\(\pm\)1.1}  & 474 \\
& QUFG  & \textbf{93.7$\pm$0.6} & \underline{90.2$\pm$0.3} & \textbf{85.5$\pm$0.6} & \textbf{66.0$\pm$1.1} & \underline{10.3$\pm$2.8} & 329.4 \\

\hline

\multirow{6}{*}{\textbf{Coauthor-Physics}} & Spectral CNN 
& 51.2$\pm$2.0 & 16.5$\pm$0.3 & 11.2$\pm$1.8 & 7.9$\pm$0.7 &  8.1$\pm$4.5  & 84.7\\
& ChebyNet 
& 65.8$\pm$3.5 & 16.8$\pm$2.3 & 12.1$\pm$2.9 & 10.9$\pm$2.4 & 9.0$\pm$3.6 & 791.4\\
& GCN 
& 94.1$\pm$0.5 & 92.7$\pm$0.4 & 84.5$\pm$1.1  & \textbf{66.4$\pm$7.0} & 30.3$\pm$18.5  &  92.7\\
& GWNN 
& 79.2$\pm$1.3 & 76.9$\pm$0.3 & 68.3$\pm$2.1 & 50.0$\pm$4.5 & 19.2$\pm$3.4  & 201.3\\
& GPRGNN  & 93.5\(\pm\)1.8   & 81.4\(\pm\)2.6  & 39.1\(\pm\)4.9  & 28.8\(\pm\)6.4   & 20.1\(\pm\)2.5  &  129.4\\
& H2GCN  & 94.2\(\pm\)1.5 & 60.5\(\pm\)2.7 & 38.4\(\pm\)6.5   & 22.1\(\pm\)3.8 & 18.2\(\pm\)2.5 &  511.2\\
& UFG
& \underline{94.3$\pm$0.4} & \textbf{93.0$\pm$0.6} & {84.7$\pm$1.7}  & 43.6$\pm$12.5 & {40.4$\pm$0.5} & 1490.5 \\

& pL-UFG$^{2.5}$ &91.7\(\pm\)0.1  &88.9\(\pm\)1.5  &\underline{86.1\(\pm\)2.4}  &{51.4\(\pm\)4.5}  &\textbf{49.8\(\pm\)3.6}  &1630 \\

& QUFG & \textbf{94.4$\pm$0.4} & \underline{92.8$\pm$0.4} & \textbf{87.0$\pm$0.4} & \underline{56.0$\pm$8.1} & \underline{43.8$\pm$2.3} &  1380.4\\
\hline
\end{tabular}
\end{adjustbox}
\end{table}

\subsection{{Robustness w.r.t. Noisy Node Features}}

Our QUFG model is designed to reduce perturbations on node features. In order to compare its superior ability with other spectral GNNs and spatial GCN baselines, the node classification task is conducted under the noisy node feature setting and accuracy results are reported. 

In the experiments, the denoising capability of the QUFG for node classification tasks is measured by adding noises to citation graphs. First, the denoising experiments are performed by including binary noises on node features on Cora, CiteSeer, Coauthor-Physics, and Coauthor-CS where node features are bag-of-words representations of documents. The binary noising mechanism involves creating a random matrix of values between 0 and 1, converting it to a binary mask based on a specified noising ratio, and then using this mask to modify the original node feature $\mathbf X$. The modified matrix is then processed to generate the final noised feature matrix, where values are either 1.0 or 0.0 based on the modifications applied by the mask. Furthermore, the PubMed dataset is used to test the model's capabilities in dealing with Gaussian noises since node features in Pubmed are described by a TF/IDF weighted word vector where each entry is a float value between 0 and 1. Similarly, a Gaussain noise matrix is initially generated according to the noise ratio and then added to the feature matrix of PubMed. The whole noising mechanism is with similar characterises compared to the previous works \citep{zhang2020gnnguard,zhang2021backdoor}.

Table \ref{table:table 6a} reports the node classification results under the binary denoising settings: binary noises with different binary noise levels: 5, 15, 25, 50, which represents what percentage of nodes are poisoned by binary noises. In addition, we also report the average running time per epoch in ms of all included models.
Meanwhile, Table~\ref{table:table Gaussian} reports the accuracy performances of various GNN models conducting node classification tasks on the Pubmed dataset under the Gaussian denoising settings. Gaussian noises are set with different levels represented by various standard deviations: 0.05, 0.1, and 0.2.  The results show that the QUFG model shows superior performance over the baseline models in nearly all the cases, which proves the QUFG model's denoising strength. However, one can check that with the same percentage of the noise level, the learning accuracy of GNNs under the data with Gaussian noise is much less than those under the binary noise, suggesting that the Gaussian noise attack has higher destructive power than the binary noise to the node features. However, when the Gaussian noise level increases, the decreasing speed of the learning accuracy slows down. We highlight that a potential explanation for this is when a higher frequency (larger standard derivation) of Gaussian noise is deployed to the graph; this will bring more effect on those nodes with higher distinctive features than others. However, given the input graph (Pubmed) is a citation graph in which connected nodes are more likely to share the same labels, the role of QUFG and other GNNs is mainly to homogenize the features between connected nodes. Therefore, those highly noised nodes are  (1) essentially with different labels compared to others, and (2) rare in the citation graphs, and this explains why there is a drop in the decreasing speed. However, fully quantifying the reason for this observation might require some additional assumptions on the feature distributions or directly simulated features with distribution assumptions; we leave this to future work. 

Moreover, although both UFG and QUFG have relatively high time consumption, their average running times are still less than 2 seconds per epoch. Lastly, one can find that serving as the elastic GNN on the quasi-framelet domain. The pL-UFG also shows good performance along with all included datasets with relatively higher computational time. We note that this is due to the fact that the so-called implicit layer is incorporated as additional computation in pL-UFG. However, our QUFG is with similar or even higher accuracy than pL-UFG without the additional computation, suggesting the effectiveness of our proposed model.

\begin{table}[H]
\caption{Node classification accuracy (\%) on Pubmed under Gaussian noise settings (Best results highlighted by bold).}
\label{table:table Gaussian}
\centering 
\begin{adjustbox}{width= 12cm,center}
\tabcolsep 2.5pt 
\begin{tabular}{|c|cccc||c|cccc|}
\hline
 Noise (\%) & Clean & 5 & 10 & 20 & & Clean & 5 & 10 & 20\\
\hline
Spectral 
& 73.9 & 37.4$\pm$3.5 & 36.7$\pm$3.2  & 36.5$\pm$3.1 & GWNN & 79.1 & 20.9$\pm$0.4  & 19.6$\pm$0.3 & 19.6$\pm$0.2 \\

ChebyNet 
& 74.4 & 37.6$\pm$3.3 & 35.7$\pm$2.8  & 36.8$\pm$3.0  &UFG 
& 79.9  & 43.3$\pm$2.1 & 42.9$\pm$1.5 & 41.3$\pm$1.9 \\

H2GCN  & 78.8   & 32.4\(\pm\)2.6   & 32.2\(\pm\)3.6    &31.3\(\pm\)4.4  & 
GPRGNN  & 79.7   & 37.2\(\pm\)3.1   & 36.8\(\pm\)2.9    & 38.1\(\pm\)4.5\\

GCN  
& 79.0 & \textbf{44.1$\pm$1.4}  & 42.6$\pm$1.4 & 41.1$\pm$2.7 & QUFG & \textbf{80.6}  
& 43.6$\pm$2.9 & \textbf{43.8$\pm$1.7} & \textbf{41.5$\pm$2.2} \\
\hline
\end{tabular}
\end{adjustbox}
\end{table}

\begin{figure}[t]
    \centering
    \includegraphics[width=1.05\linewidth]{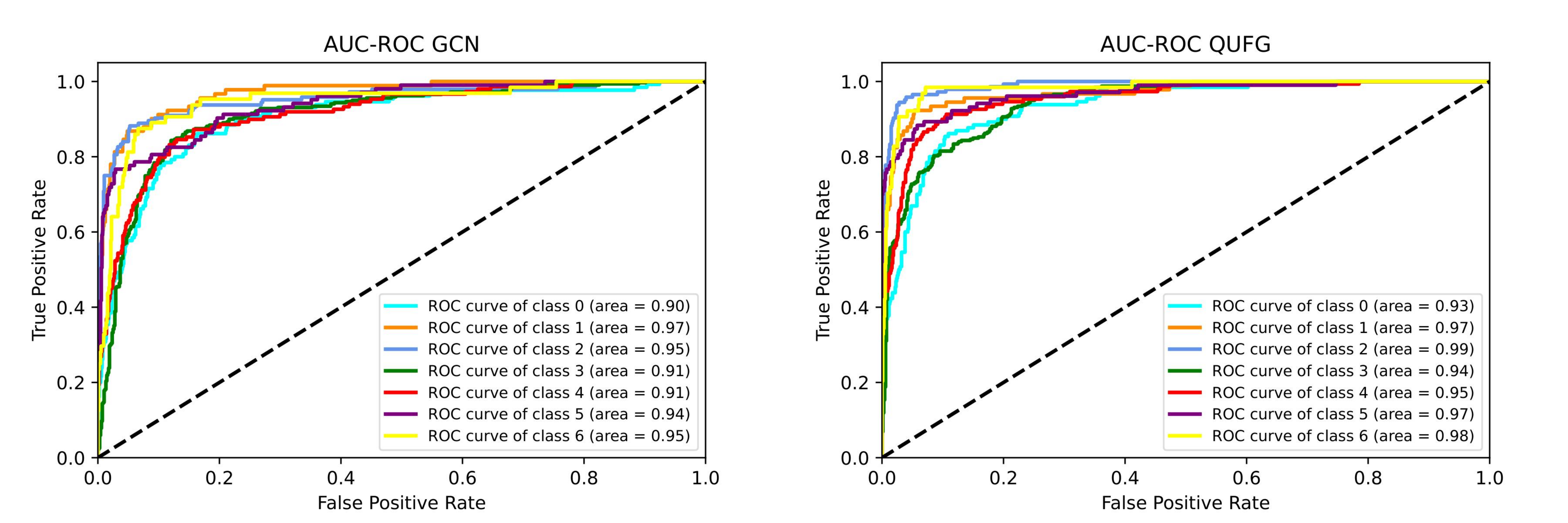}
    \caption{AUC-ROC between GCN and QUFG's performance on Cora (7 node classes). One can check that with higher average accuracy, the AUCs of QUFG (right) are also bigger than those in GCN (left).}
    \label{auc-roc}
\end{figure}
 Furthermore, given in this node classification task, both our model and other baselines can be served as classifiers. Accordingly, we include Figure.~\ref{auc-roc} as the area under the ROC curve (AUC-ROC) for the learning results between GCN and QUFG on Cora. 

Overall, the results show that all spectral graph convolutions are affected by noises. Nevertheless, our QUFG model achieves higher classification accuracy over baseline models, especially under the denoising settings. This proves the QUFG model's robustness strength w.r.t. noisy node features.

\subsection{Robustness under Adversarial Attacks}
Adversarial attacks \citep{SunDouYangZhangWangYuHeLi2023} on graph data involve manipulating the input graph to deceive a machine learning model, leading to misclassification or erroneous predictions. In the context of graph data, there are four major types of adversarial attacks, such as \textit{graph perturbation}, \textit{feature manipulation}, \textit{structure-based attacks} and \textit{structural perturbation}. These methods collectively contribute to the development of effective adversarial attacks on graph data, fostering research in robust machine learning models and defense mechanisms for graph-based tasks.

Under the setting of untargeted adversarial attacks on graph structures, we test and compare the performance of the QUFG model with that of other GCN models like Elastic GNNs \citep{liu2021elastic} to evaluate its robustness. We conduct experiments on perturbated datasets\footnote{\small\url{https://github.com/ChandlerBang/Pro-GNN}} including Cora, Citeseer and PubMed provided by \citet{jin2020graph}. These perturbated datasets are created by implementing the MetaAttack \citep{zugner2018adversarial} on clean graphs. QUFG models and baseline models are all trained on the attacked graphs respectively and are compared based on their node classification accuracy against the non-targeted adversarial attacks. Table~\ref{table:table adver} reports the model's accuracy performance. 

\begin{table}[H]
\caption{Node classification accuracy (\%) under adversarial graph attacks. The best results are highlighted by bold.}
\label{table:table adver}
\centering
\begin{adjustbox}{width= 10.5cm,center}
\begin{tabular}{|c|c |c c c c|}
\hline
Datasets & Noise Level (\%) & 5 & 10 & 15 & 20 \\
\hline
\multirow{6}{*}{Cora} & Spectral CNN  
& 49.2$\pm$0.9  & 49.5$\pm$1.0  & 51.6$\pm$1.0 & 53.3$\pm$0.8\\
& GCN  
& 76.7$\pm$0.8  & 67.8$\pm$1.0 & 61.0$\pm$0.9 & 11.9$\pm$4.5\\
& Elastic GNN  
& 82.2$\pm$0.9  & 78.8$\pm$1.7 & \textbf{77.2$\pm$1.6} & 70.5$\pm$1.3 \\
& UFG  
& 81.4$\pm$0.5  & 78.8$\pm$0.8 & 74.6$\pm$1.1 & 66.5$\pm$1.9 \\
& QUFG & \textbf{82.8$\pm$0.5} & \textbf{79.3$\pm$0.7} & 74.6$\pm$0.6 & \textbf{71.0$\pm$1.0} \\
\hline
\multirow{6}{*}{Citeseer} & Spectral CNN  
& 50.5$\pm$1.7 & 49.7$\pm$1.0 & 53.3$\pm$1.2 & 51.3$\pm$0.7  \\
& GCN  
& 70.9$\pm$0.6 & 67.6$\pm$0.9 & 64.5$\pm$1.1 & 62.0$\pm$3.5 \\
& Elastic GNN  
& 73.3$\pm$0.6  & 72.6$\pm$0.4 & 71.9$\pm$0.7 & 64.9$\pm$1.0 \\
& UFG  
& 74.4$\pm$0.5 & 72.2$\pm$0.8  & 70.5 $\pm$ 0.9 & 63.9 $\pm$ 1.4  \\
& QUFG  & \textbf{75.1$\pm$0.4} & \textbf{73.8$\pm$0.3} & \textbf{72.0$\pm$0.4} & \textbf{67.0$\pm$0.4} \\
\hline
\multirow{6}{*}{Polblogs} & Spectral CNN  
& 47.6$\pm$1.4 & 44.2$\pm$0.7 & 50.2$\pm$2.3  & 50.0$\pm$0.6  \\
& GCN  
& 73.1$\pm$0.8 & 70.7$\pm$1.1 & 65.0$\pm$1.9 & 51.3$\pm$1.2 \\
& Elastic GNN  
& 83.0$\pm$0.3  & 81.6$\pm$0.3 & 78.7$\pm$0.5 & 77.5$\pm$0.2 \\
& UFG  
&  90.0$\pm$0.7 & 87.8$\pm$ 0.4  & 86.7$\pm$0.7 & 86.1$\pm$0.6  \\
& QUFG  & \textbf{90.7$\pm$0.3} & \textbf{88.6$\pm$0.3} & \textbf{87.6$\pm$0.5} & \textbf{87.0$\pm$0.3} \\
\hline
\end{tabular}
\end{adjustbox}
\end{table}
Although QUFG seems to work directly on graph signals, its filtering capability and convolution operator also enables the robustness of our model under adversarial attacks. In the process, some attacked edges may be automatically seen as insignificant since the features of their connected nodes are recognized by QUFG as noise or unimportant signals.  It is worth pointing out that QUFG is not designed explicitly for defending against adversarial attacks, instead its robust performance naturally comes from adaptively selecting appropriate frequency components when training against the node targets/labels for node classification tasks. This contrasts with recent works such as RGCN (Zhu et al 2023) and GNNGuard (Zhang and Zitnik 2020) that focus on handling different types (i.e., target and untargeted) of adversarial attacks specifically.

\subsection{The Effect of Cut-off Strategies} \label{Sec:4.4}

One of the main reasons that the QUFG shows strengths in its robustness to noise and adversarial attacks is its high flexibility in cutting off unwanted spectral information and preserving uncommon true graph signals. In signal processing, many models act as low-pass filters, treat high-frequency signals as noises, and thus, ignore the true values hidden within high-frequency patterns. To further demonstrate the denoising ability of the QUFG, especially in relevant graph high-frequency components, the following experiments have been conducted on QUFG against the state-of-the-art UFG framelets.

To simulate the high-frequency noises, we pick up the eigenvectors of the normalized Laplacian $\mathbf L$ corresponding to the larger eigenvalues. Assume that $\mathbf L = \mathbf U\boldsymbol{\Lambda}\mathbf U^\top$ where  eigenvalues $\boldsymbol{\Lambda}$ are in a decreasing order, then the noises to be injected into the node signal are given
$
\mathbf n = \mathbf U[:,1:F]\mathbf w
$
where $F$ is the number of high-frequency components with two choices ($F=100$ or $500$) and $\mathbf w$ in shape $[F,1]$ is the noise levels, randomly chosen from normal distribution with standard deviation levels $nl=5.0$, 10.0 and 20.0, respectively. We will use the benchmark parameters as in \citep{zheng2021framelets}, by setting two GCN layers, hidden feature $nhid = 16$, Chebyshev order $n=3$ (The results for $n=2$ demonstrate a similar pattern, see Appendix C.), Soft threshold for denoising, weight decay 0.01 and learning rate 0.005 with epoch number 200.

As a case study on Cora, we consider two high-frequency components cut-off strategies: (a) \texttt{PartialCutoff}: In \eqref{Eq:Convolution} setting $\mathcal{W}_{K,0}=0$ to remove the highest frequency; and (b) \texttt{FullCutoff}: In \eqref{Eq:Convolution} setting $\mathcal{W}_{K,l}=0$ ($\ell=0, 1, ..., L$) to remove all high frequencies.

\begin{table}[H]
\centering
\caption{Denoising results: node classification accuracy (\%) on Cora with varying noise levels (20 in 1st row, 10 in 2nd row and 5 in 3rd row.}
\label{Table4} 
\begin{adjustbox}{width= 12cm,center}
\begin{tabular}{|c|c|c|c|c|c|c|c|}\hline 
   \multicolumn{2}{|c|}{PartialCutoff(F=100)} &\multicolumn{2}{c|}{FullCutoff(F=100)}  &  \multicolumn{2}{c|}{PartialCutoff(F=500)} &\multicolumn{2}{c|}{FullCutoff(F=500)}\\ \cline{1-8}
 QUFG & UFG & QUFG & UFG  & QUFG & UFG & QUFG & UFG\\ \hline
   82.1{\scriptsize $\pm$0.3} &52.7{\scriptsize $\pm$1.9} &82.1{\scriptsize $\pm$0.3} &71.1{\scriptsize $\pm$0.7}&80.4{\scriptsize $\pm$0.8}&26.8{\scriptsize $\pm$0.7} &80.5{\scriptsize $\pm$0.5}&29.1{\scriptsize $\pm$1.5}\\
\hline    
 81.9{\scriptsize $\pm$0.3}&66.5{\scriptsize $\pm$0.7}&82.0{\scriptsize $\pm$0.3}&76.1{\scriptsize $\pm$0.6}&80.7{\scriptsize $\pm$0.6}&28.5{\scriptsize $\pm$1.1}&81.0$\pm$0.6&32.8{\scriptsize $\pm$0.9}\\
\hline
82.2{\scriptsize $\pm$ 0.2}&73.3{\scriptsize $\pm$0.6}&81.8{\scriptsize $\pm$0.9}&78.4{\scriptsize $\pm$0.6} &81.5{\scriptsize $\pm$0.8}&32.1{\scriptsize $\pm$0.6}&81.6{\scriptsize $\pm$0.8}&43.5{\scriptsize $\pm$1.0} \\
\hline
\end{tabular}
\end{adjustbox}
\label{table:table 4}
\end{table}

Experimental results under both partial cutoff and full cutoff settings are reported in Table \ref{table:table 4}. In both \texttt{PartialCutoff} and \texttt{FullCutoff} settings, with a group of entropy modulation functions, the QUFG model demonstrates its accuracy performance of very similar patterns across all cases, while the accuracy performance of the UFG model looks significantly different for two cut-off strategies. This is because the Entropy framelet has a strong capability to fully cut off the high frequencies no matter at what noise scales, while as the UFG framelet is designed from the spatial domain, there is no control over what frequency components should be regulated. Furthermore, although UFG can denoise relevant high-frequency components (in the case of $F=100$), it almost fails in identifying mild high-frequency components (in the case of $F=500$).

\begin{figure}[t]
    \centering
    \includegraphics[width=0.85\linewidth]{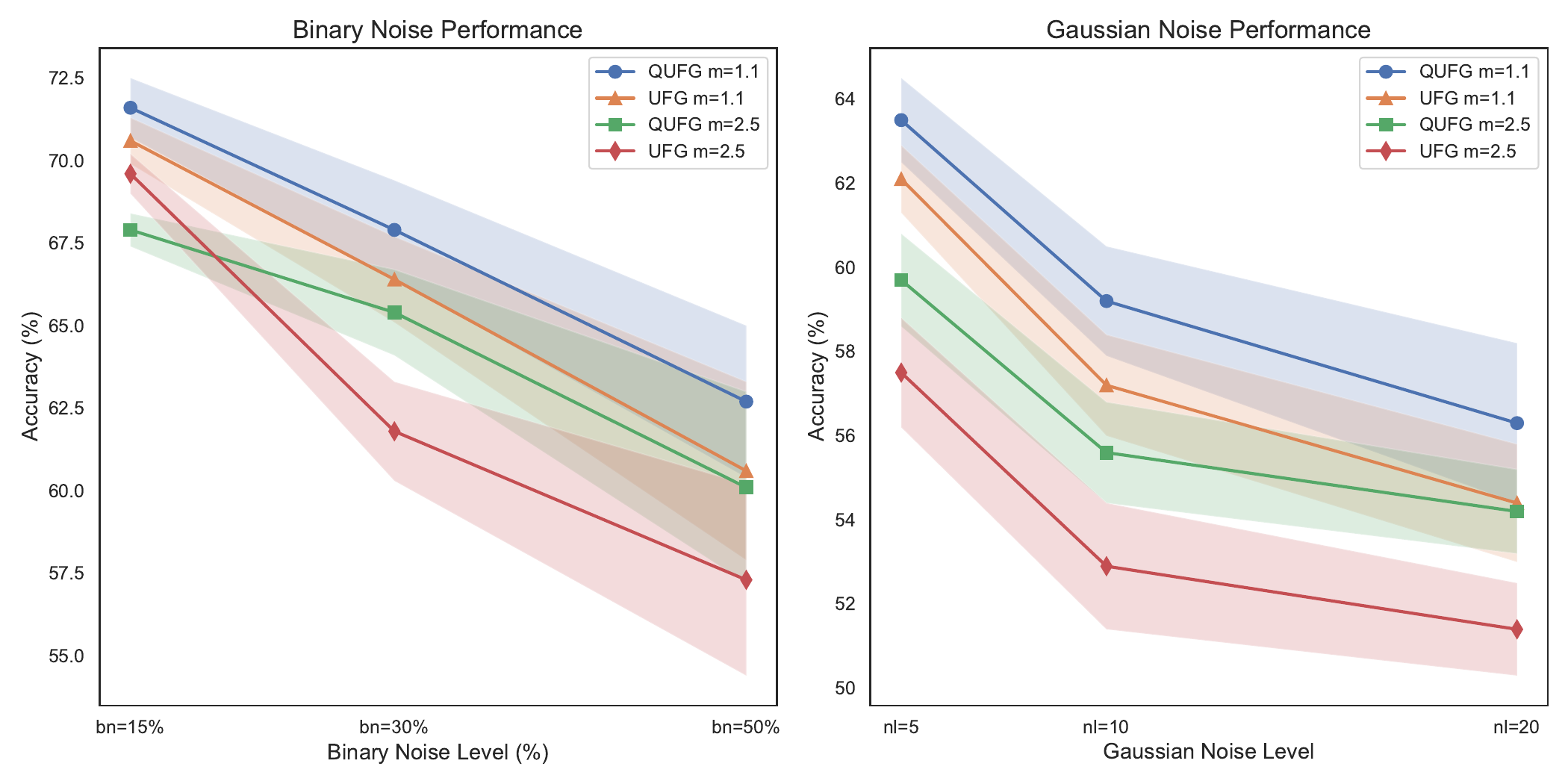}
    \caption{Denoising performance of the model (UFG and QUFG) via different coarse scale levels.}
    \label{fig:denoising_sensiivity}
\end{figure}

\subsection{Experiments on Parameter Sensitivity}
\subsubsection{Sensitivity on the Coarse Scale Level} In this test, we conduct experiments on the parameter sensitivity. We consider several hyperparameters. The first one is the coarsest scale level $m$, which is the smallest quantity that satisfies $2^{-m}\lambda_N \leq \pi$. Another hyperparameter that we consider is the quantity of $\alpha$ which controls modulation power at both the lowest and highest frequencies used in the entropy filtering functions. Given the scale level $m$ is used for both UFG and QUFG to ensure the range of the eigenvalues is within the domain of the Chebyshev polynomial, we compare our QUFG to the UFG model in terms of different levels of noise. Specifically, 
we test the effect of denoising with respect to two types of noises: (1) spreading $bn =15\%$, $30\%$ and $50\%$ poison with random binary noise; (2) injecting Gaussian white noise with standard deviation at levels $nl=5.0$, 10.0 and 20.0 into all the features.  
The filters in the UFG model and QUFG model are conducted and compared by scale $m = 2.5$ and scale $m=1.1$ respectively.  

From the Fig.~\ref{fig:denoising_sensiivity} one can check that QUFG always outperforms UFG along with all noise levels and scales, suggesting the enhancing power of the QUFG is essentially from the carefully designed filtering functions rather than the parameter tuning. In addition, we also observed that the results of $m = 1.1$ is always better than the results $m =2.5$, this suggests a stronger scaling (i.e., $m>2$) of the Laplacian eigenvalues is unnecessary for regulating the eigenvalues of the graph to fit the domain of Chebyshev polynomial in terms of the denoising tasks. However, to fully explore the effect of $m$ on denoising tasks quantitatively might require one to explicitly illustrate the changes of the dynamic of QUFG in terms of $m$ \citep{han2022generalized}; we leave this to future work. Lastly, both QUFG and UFG show increasing standard deviations in high noise levels, suggesting a potential limitation on model's stability via highly noised data, and to resolve the problem, one may need to manually control models' energy dynamic \citep{han2023continuous} and this may out of the scope of our work, we omit it here.

\subsubsection{Sensitivity on the Modulation power}\label{sensitivity_on_alpha}
Followed by the sensitivity test on $m$, we test the models' sensitivity on $\alpha$ in the entropy filtering functions. Recall that in our proposed entropy filtering functions, we let $0 \leq \alpha \leq 1$, and we fixed $\alpha = 0.75$ in all of our experiments. In general, $\alpha$ modulates the amplitude and shape of the parabolic function $g_1({\xi})$  directly impacting the filter's strength and distribution of the signal's energy. Higher values of $\alpha$ result in greater modulation power, allowing for a broader distribution of the filtered response, while lower values concentrate the filter’s effect. This modulation capability is essential for adapting the filter’s response to different signal characteristics, thereby enhancing the filtering process. The Figure below shows the effect of $\alpha$ on model's performance across various datasets. 
 
\begin{figure}[htp]
    \centering
    \includegraphics[width=0.5\textwidth]{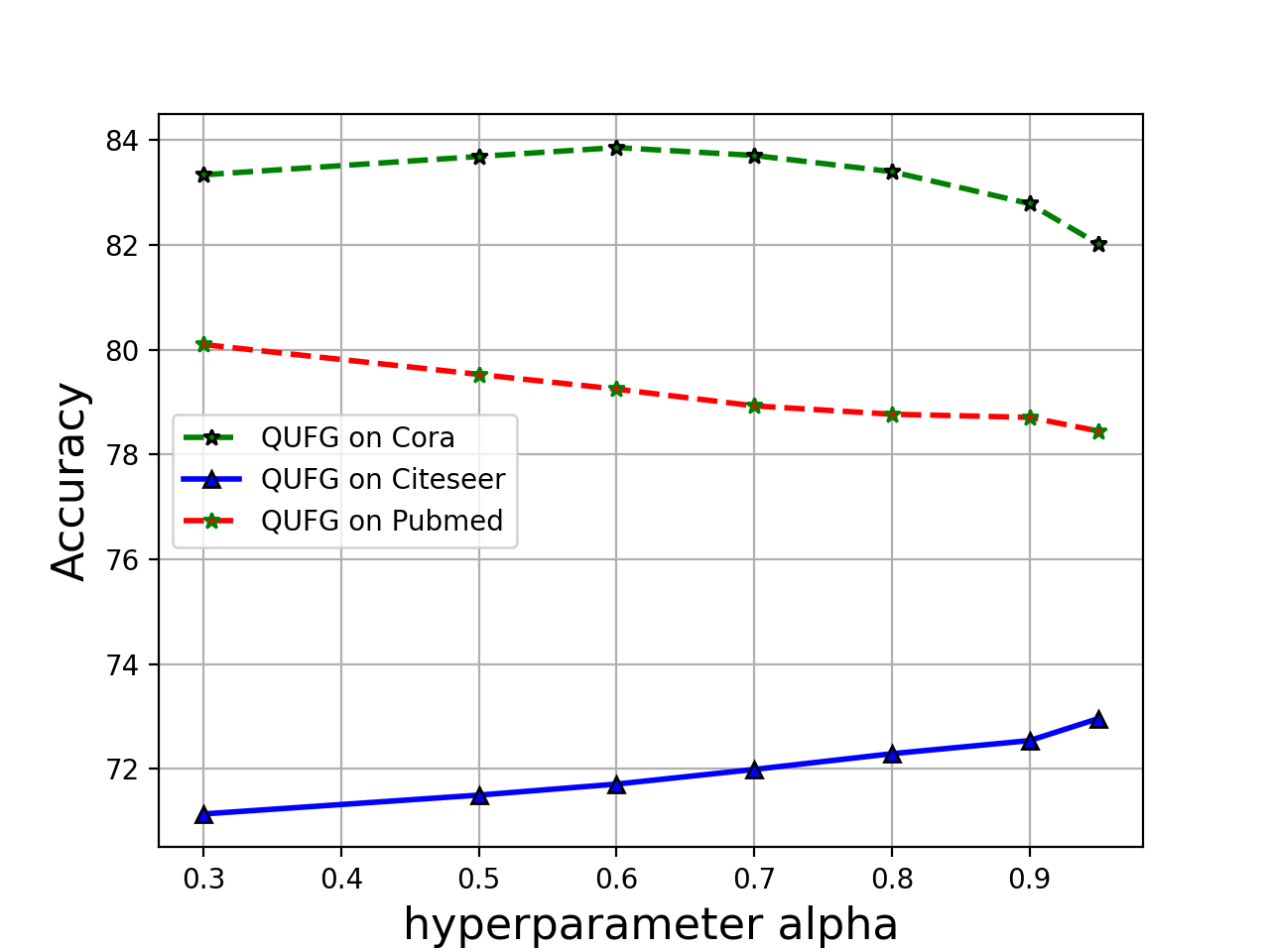}
    \caption{For datasets Cora, Citeseer, and Pubmed, the accuracy of the QUFG GCN model with respect to different $\alpha$ values in entropy modulation function ranging from 0.3 to 0.95.}\label{hyperparameter_alpha}
\end{figure}

From Fig.~\ref{hyperparameter_alpha}, one can check that for both Cora and PubMed datasets, the accuracy of QUFG dropped via the increase of $\alpha$, whereas in Citeseer, the accuracy increased. We note that this observation could be due to the notion of so-called graph homo/heterophily, which we introduced in Section \ref{sec:adaption_power}. The so-called homophily index has been widely applied in recent studies to measure the degree of homophily and heterophily. It is known that the homophily indices for Cora and Pubmed are 0.8 and 0.79, respectively, whereas for Citeseer, it is only 0.71. This suggests that compared to Cora and Pubmed, Citeseer requires GNNs to learn a broader distribution of the filtered responses, and this directly verifies our observation where a higher $\alpha$ gives QUFG better results in Citeseer and lower $\alpha$ provides the best results for Cora and Pubmed.

\section{Concluding Remarks and Further Research Directions}\label{Sec:5}
This paper extends the classic multiscale graph framelets by introducing a group of newly designed filtering functions. These functions are proven to inherit many classic properties of graph framelets, such as perfect reconstruction and low computational complexity. Furthermore, we have empirically shown that the QUFG equipped with the proposed filtering functions is more robust to different types of noises than the graph framelet, suggesting a stronger denoising and adaptive power. Our investigation into the robustness of the QUFG model to noise positions our work alongside other studies on framelet properties, including robustness to over-smoothing (OSM) issues and strong adaptation to heterophilic graphs. Most of the research in this field resolved these problems (OSM and heterophily adaptation) through spatial methods, such as adjacency matrix reweighting, rewiring, and adding source terms, yet rare of them considered mitigating these problems via spectral domain. Thus the potential future study direction is to explore the effect of spectral filtering function on these computational problems in GNNs.

\clearpage 
\bibliographystyle{spbasic}
\bibliography{reference}

\end{document}